\documentclass{article}

\usepackage[final,main]{neurips_2026}
\usepackage{url}
\usepackage[linesnumbered,ruled,vlined]{algorithm2e}
\usepackage[utf8]{inputenc}
\usepackage[T1]{fontenc} 
\usepackage{graphicx}
\usepackage{amsmath}
\usepackage{amsfonts}
\usepackage{amsthm}
\usepackage{nicefrac}       
\usepackage{microtype}      
\usepackage{booktabs}
\usepackage{mathtools} 
\usepackage{multirow}
\usepackage{wrapfig}
\usepackage{array}
\usepackage{tikz,pgfplots}
\usepackage{subcaption}
\usepackage{framed}
\usepackage{siunitx}
\usepackage{xcolor}
\usepackage{xspace}
\usepackage{pgfplots}
\pgfplotsset{compat=1.18}
\usepackage{tikz}
\usepackage{pgfplotstable}
\usepgfplotslibrary{groupplots}

\usepackage{hyperref}

\usepackage[capitalize,noabbrev]{cleveref}

\theoremstyle{plain}
\newtheorem{theorem}{Theorem}[section]

\newtheorem{lemma}[theorem]{Lemma}

\theoremstyle{definition}

\newtheorem{assumption}[theorem]{Assumption}
\theoremstyle{remark}

\usepackage[textsize=tiny]{todonotes}

\pgfplotsset{compat=1.18}

\newcolumntype{P}[1]{>{\centering\arraybackslash}p{#1}}

\def\framework{\texttt{AeroChem}\xspace}
\definecolor{myred}{HTML}{D62728}    
\definecolor{myblue}{HTML}{1F77B4}   
\definecolor{myorange}{HTML}{FF7F0E} 
\definecolor{mygreen}{HTML}{2CA02C}  
\definecolor{mypurple}{HTML}{9467BD} 
\definecolor{mygray}{HTML}{7F7F7F}   
\newcommand{\Sm}{\mathbf{S}}
\newcommand{\Sprior}{\Sm_{t}^{-}}
\newcommand{\Spost}{\Sm_{t}^{+}}
\newcommand{\Spostprev}{\Sm_{t-1}^{+}}
\newcommand{\Pprior}{\mathbf{P}_{t}^{-}}
\newcommand{\Ppost}{\mathbf{P}_{t}^{+}}
\newcommand{\Zobs}{\mathbf{Z}_{\mathrm{obs},t}}
\newcommand{\Robs}{\mathbf{R}_{\mathrm{obs},t}}
\newcommand{\Kt}{\mathbf{K}_{t}}

\newcommand{\Fdrift}{\mathbf{F}_{\mathrm{drift}}}

\newcommand{\eps}{\boldsymbol{\epsilon}}

\newcommand{\E}{\mathbb{E}}
\newcommand{\R}{\mathbb{R}}

\newcommand{\norm}[1]{\lVert #1 \rVert}
\newcommand{\Phistep}{\Phi_{\Delta t}}
\newcommand{\err}{\mathbf{e}_{t}}
\newcommand{\errprev}{\mathbf{e}_{t-1}}
\newcommand{\wt}{\mathbf{w}_{t}}
\newcommand{\nut}{\boldsymbol{\nu}_{t}}
\newcommand{\nutbar}{\bar{\boldsymbol{\nu}}_{t}}
\newcommand{\nuttilde}{\tilde{\boldsymbol{\nu}}_{t}}

\newcommand{\Filt}{\mathcal{F}}

\newcommand{\LinOp}{\mathrm{Linear}}
\newcommand{\Strue}{\Sm_{t}^{*}}
\newcommand{\Strueprev}{\Sm_{t-1}^{*}}

\title{\framework: Closed-loop Physics-Informed State Space Modeling for Long-term Chemically-Reactive \\Air Quality Forecasting}

\author{%
  Kien Le$^{*}$ \\
  VinUniversity\\
  Vietnam \\
  {\small\texttt{kien.lt@vinuni.edu.vn}} \\
  \And
  Tung Kieu$^{*}$ \\
  Aalborg University\\
  Denmark \\
  {\small\texttt{tungkvt@cs.aau.dk}} \\
  \And
  Dinh Duc Nha Nguyen \\
  CECS, VinUniversity\\
  Vietnam \\
  {\small\texttt{nha.ndd@vinuni.edu.vn}} \\
  \AND
  Cuong Do \\
  CECS, VinUniversity\\
  Vietnam \\
  {\small\texttt{cuong.dd@vinuni.edu.vn}} \\
  \And
  Bin Yang \\
  East China Normal University\\
  China \\
  {\small\texttt{byang@dase.ecnu.edu.cn}} \\
  \And
  Nguyen Phi Long$^{\dagger}$ \\
  {\small SEEE, Hanoi University of Industry}\\
  Vietnam \\
  {\small\texttt{long.nguyen04@haui.edu.vn}}
}

\begin{document}

\maketitle
\renewcommand{\thefootnote}{\fnsymbol{footnote}}
\footnotetext[1]{Equal contribution.}
\footnotetext[2]{Corresponding author.}
\renewcommand{\thefootnote}{\arabic{footnote}}
\setcounter{footnote}{0}


\begin{abstract}
Accurate long-term air quality forecasting is challenging due to long-range transport, meteorological variability, and nonlinear chemical reactions underlying secondary pollutant formation. Purely data-driven models often suffer from long-horizon drift and lack physical consistency, while existing physics-informed approaches commonly follow open-loop designs that fuse neural and physical components only at the output level. We propose \framework, a closed-loop physics-informed state space framework for chemically reactive air quality forecasting. \framework adopts a two-branch co-correction design: a UDE-based process branch provides a structured but imperfect physical prior, while a Mamba-based virtual measurement branch captures long-range temporal context and produces probabilistic virtual observations. Neural Kalman Fusion couples the two branches by correcting the latent physical state at each forecasting step and feeding the posterior state back into the simulator, reducing long-horizon drift. We further introduce \textbf{HanoiAir}, a real-world dataset integrating pollutant concentrations, meteorology, and emission inventories. Experiments on \textbf{HanoiAir} and a public benchmark show that \framework improves long-horizon stability and sudden-change forecasting.
\end{abstract}

\section{Introduction}
\label{sec:introduction}

Air pollution is a major global health and environmental challenge, motivating accurate air quality forecasting for public policy, emission control, and public-health protection~\cite{DBLP:conf/kdd/HsiehLZ15,DBLP:conf/kdd/ZhengYLLSCL15}. However, forecasting pollutant concentrations remains difficult because atmospheric dynamics are governed by coupled emission sources, meteorological conditions, transport processes, and chemical reactions~\cite{DBLP:journals/tgrs/ZhuXFYLZZWWS23}. This challenge is especially pronounced for secondary pollutants, such as ozone (O\textsubscript{3}) and secondary components of PM\textsubscript{2.5}, which arise from nonlinear photochemical reactions involving precursors such as NO\textsubscript{x} and VOCs under meteorological drivers such as solar radiation, temperature, and humidity. These processes evolve over multiple hours or days, requiring both long-range temporal modeling and physical consistency.

Traditional Chemical Transport Models (CTMs)~\cite{DBLP:conf/igarss/SunXBCWJJZ22,DBLP:journals/remotesensing/WenSLMWWXS21}, as illustrated in Fig.~\ref{fig:impacts-on-air-quality}, encode atmospheric transport, diffusion, and chemical reactions through differential equations derived from physical principles. Although interpretable, CTMs rely on simplified assumptions, incomplete reaction mechanisms, uncertain boundary conditions, and expensive solvers. In contrast, data-driven deep learning models learn predictive patterns directly from observations and often achieve strong short-term accuracy, but they may violate physical constraints, generalize poorly under distribution shifts, and accumulate errors during long-horizon autoregressive rollout.

\begin{wrapfigure}{r}{0.35\linewidth}
    \centering
    \includegraphics[width=0.95\linewidth]{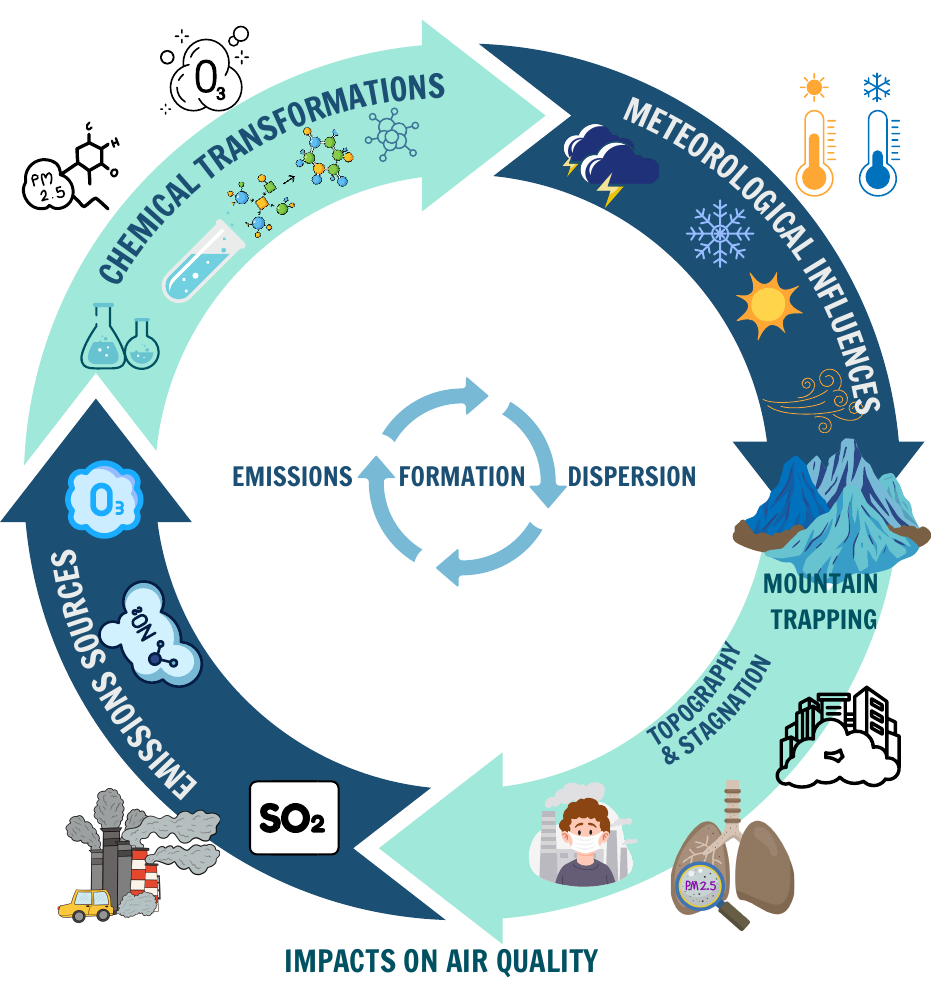}
    \caption{\small Impacts on Air Quality.}
    \label{fig:impacts-on-air-quality}
    \vspace{-10pt}
\end{wrapfigure}

Physics-Informed Machine Learning (PIML)~\cite{DBLP:journals/jcphy/RaissiPK19,DBLP:conf/colt/DoumecheBBB24} aims to combine physical modeling with neural learning, but existing air quality forecasting methods remain limited. Many oversimplify atmospheric chemistry as linear source--sink terms or generic neural residuals, which is insufficient for nonlinear and condition-dependent secondary pollutant formation. More importantly, most hybrid architectures follow an open-loop design, where physical and data-driven branches are executed separately and fused only at the output level. Such designs may correct the final prediction, but they cannot correct the latent physical state propagated to future steps, allowing errors from misspecified physics, unobserved emissions, and stochastic atmospheric disturbances to accumulate.

In this work, we argue that long-term chemically reactive forecasting requires a closed-loop design, where physics and data-driven evidence correct each other during rollout. The physical branch should provide a structured dynamical prior, while the data-driven branch should provide observation-conditioned evidence for correcting the evolving state. Crucially, this correction should occur in latent state space before the next simulation step, rather than only at the final output. This view turns long-horizon forecasting into recurrent posterior-state inference under imperfect physical dynamics.

To realize this idea, we propose \framework, a closed-loop physics-informed state space framework for long-term chemically reactive air quality forecasting. \framework couples two interacting branches through Neural Kalman Fusion. The first is a UDE-based process branch that acts as a differentiable simulator and explicitly separates transport dynamics from chemical dynamics, using diffusion, wind-driven advection, and a meteorology-gated neural reaction operator. The second is a Mamba-based virtual measurement branch that captures long-range temporal dependencies from historical observations and produces probabilistic virtual measurements. At each forecasting step, Neural Kalman Fusion combines the process prior and virtual measurement into a posterior latent state, which is recursively fed back into the process branch. This closed-loop co-correction mechanism reduces long-horizon drift and improves robustness in open, chemically reactive environments by allowing data-driven evidence and physics-informed dynamics to correct each other during rollout.

We evaluate \framework on two real-world air quality datasets, including \textbf{HanoiAir}, a newly constructed regional dataset integrating pollutant concentrations, meteorological variables, and emission inventories. Experiments show that \framework is particularly effective in long-horizon forecasting and sudden pollution-change scenarios, where open-loop and purely data-driven models often suffer from drift or over-smoothing. Ablation studies further demonstrate the importance of closed-loop feedback, meteorology-gated reaction modeling, and neural state correction.

In summary, this work makes the following contributions: (i) We formulate long-term chemically reactive air quality forecasting as closed-loop latent-state inference under imperfect physical dynamics. (ii) We propose \framework, a two-branch physics-informed state space framework that couples a UDE-based process branch and a Mamba-based measurement branch through Neural Kalman Fusion. (iii) We introduce a meteorology-gated neural reaction mechanism to model nonlinear secondary pollutant formation conditioned on emissions and meteorological drivers. (iv) We construct and release \textbf{HanoiAir}, a comprehensive regional air quality dataset containing pollutant concentrations, meteorological variables, and emission inventories. (v) We demonstrate that \framework improves long-horizon forecasting stability and sudden-change prediction across real-world air quality datasets.

\section{Preliminaries}
\label{sec:preliminaries}

\vspace{-5pt}
\paragraph{Sensor Network.}
We model an air quality monitoring system as a weighted static graph $\mathcal{G} = (\mathcal{V}, \mathcal{E}, \mathbf{A})$, where $\mathcal{V} = \{v_1, \dots, v_N\}$ denotes the set of $N$ monitoring stations, $\mathcal{E}$ represents the set of edges, and $\mathbf{A} \in \mathbb{R}^{N \times N}$ is a weighted adjacency matrix encoding spatial correlations between stations, such as geographic proximity.
Each station $v_i$ is associated with fixed spatial coordinates $\mathbf{p}_i \in \mathbb{R}^2$.


\vspace{-5pt}
\paragraph{Multivariate Observations.}
At each discrete time step $t$, the observed system state is denoted as $\mathbf{X}_t \in \mathbb{R}^{N \times D_{\text{in}}}$, where each station is associated with a $D_{\text{in}}$-dimensional feature vector.
We decompose $\mathbf{X}_t$ into three components:
    (1) \textbf{Pollutant concentrations} $\mathbf{C}_t$: Pollutant concentrations are the primary forecast targets, captured via air pollutant variables, including target pollutants (e.g., PM\textsubscript{2.5}) and related co-pollutants (e.g., O\textsubscript{3}, NO\textsubscript{x});
    (2) \textbf{Meteorological variables} $\mathbf{M}_t$: environmental drivers including scalar fields (e.g., temperature, humidity, solar radiation) and vector fields (e.g., wind speed and direction);
    (3) \textbf{Emission rates} $\mathbf{E}_t$: emission inventories describe source-intensity inputs, i.e., estimated pollutant or precursor releases into the atmosphere, including NO\textsubscript{x}, VOCs, SO\textsubscript{2}, and primary PM\textsubscript{2.5}.
Here, $\mathbf{C}_t$ represents pollutant levels already present in the atmosphere and serves as the main prediction target, whereas $\mathbf{E}_t$ represents external emission forcing that contributes to future pollutant formation and transport.
Given an input window of length $T_{\text{in}}$, the historical observation sequence is denoted as $\mathbf{X}_{1:T_{\text{in}}} \in \mathbb{R}^{T_{\text{in}} \times N \times D_{\text{in}}}$.

\vspace{-5pt}
\paragraph{Problem Formulation.}
Given the sensor graph $\mathcal{G}$ and historical observations $\mathbf{X}_{1:T_{\text{in}}}$, our objective is to learn a probabilistic forecasting function $\mathcal{F}_\Theta$ that estimates the posterior distribution of future observations over a prediction horizon $T_{\text{out}}$:
\(
    p\!\left(\mathbf{X}_{T_{\text{in}}+1:T_{\text{in}}+T_{\text{out}}}\mid\mathbf{X}_{1:T_{\text{in}}}, \mathcal{G}\right).
    \label{eq:prob_forecasting_function}
\)
This formulation enables both point prediction and uncertainty estimation for spatiotemporal air quality forecasting.  
\section{Methodology}
\label{sec:methodology}

\begin{figure*}[t]
    \centering
    \includegraphics[width=1.0\linewidth]{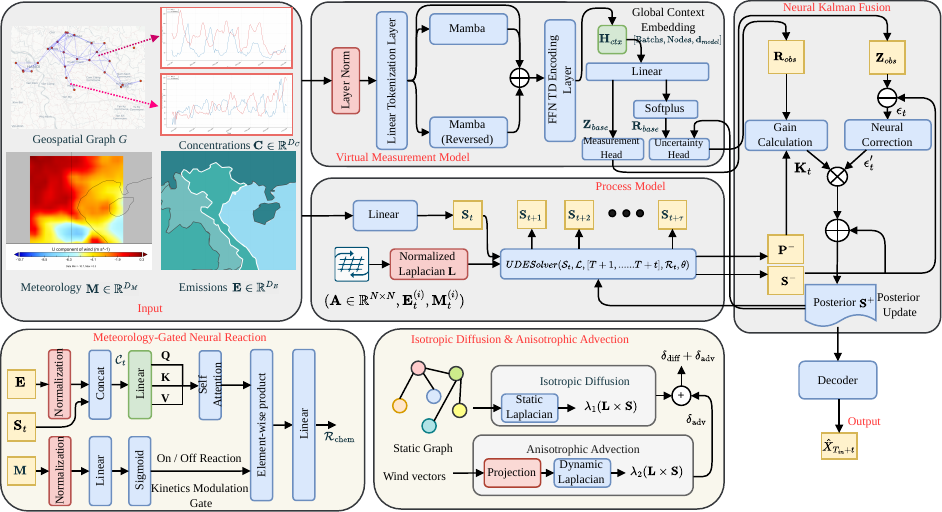}
    \caption{\small \framework architecture. A \textrm{S-Mamba}-based virtual measurement branch generates adaptive virtual observations, while a UDE-based process branch models transport and meteorology-gated neural reactions as a structured physical prior. Neural Kalman Fusion couples the two branches through a learnable correction term, producing a posterior state that is recursively fed back into the UDE. This closed-loop neural state correction allows virtual measurements and physics-informed dynamics to correct each other, mitigating long-horizon drift.}
    \label{fig:framework-overview}
    \vspace{-17pt}
\end{figure*}


To address the probabilistic forecasting problem defined in Sec.~\ref{sec:preliminaries}, we propose \framework, a closed-loop physics-informed state space framework for long-term air quality forecasting. Unlike prior hybrid approaches that loosely couple physical and data-driven components, \framework follows a \emph{closed-loop neural state correction} paradigm. At each forecasting step, a physics-informed process branch provides a structured but imperfect prior over the latent atmospheric state, while a virtual measurement branch produces data-driven virtual observations from historical context and the current rollout state. Neural Kalman Fusion couples the two branches by correcting the latent physical state before it is propagated to the next step. For clarity, we summarize the complete closed-loop forecasting procedure in Appx.~\ref{app:subsec:algorithm}.


\subsection{Framework Overview}
At each forecasting step, \framework integrates three tightly coupled components:
(1) a \textbf{Virtual Measurement Branch} that extracts long-range temporal context from historical observations and produces probabilistic virtual measurements;
(2) a \textbf{Process Branch} formulated as a stochastic Universal Differential Equation (UDE) that models transport and chemically reactive dynamics as a structured physical prior; and
(3) a \textbf{Neural Kalman Fusion} module that fuses the virtual measurement with the UDE prior to produce a corrected posterior latent state.
The posterior state is recursively fed back into the process branch, forming a closed-loop co-correction system that mitigates error accumulation in long-horizon forecasting.
The overall architecture is illustrated in Fig.~\ref{fig:framework-overview}.



\subsection{Virtual Measurement Branch: Context-Conditioned Virtual Observations}

The virtual measurement branch serves as a \emph{context-conditioned virtual sensor} that generates probabilistic virtual observations from historical data.
Unlike a standard observation model, it does not directly map the latent state to real measurements; instead, it produces adaptive virtual observations that guide the correction of the latent physical state during rollout.
Given the input sequence $\mathbf{X}_{1:T_{\text{in}}}$, we employ \textbf{S-Mamba}~\cite{DBLP:journals/ijon/WangKFWYZWZ25}, a selective state space model, to efficiently capture long-range temporal dependencies. The encoder produces a global context embedding
$\mathbf{H}_{\text{ctx}} \in \mathbb{R}^{N \times d_{\text{model}}}$.
We then convert this context into a base virtual observation and its uncertainty via two projection heads:
\begin{equation}
\mathbf{H}_{\text{ctx}} = \text{S-Mamba}(\mathbf{X}_{1:T_{\text{in}}}), \quad\quad 
\mathbf{Z}_{\text{base}}, \mathbf{R}_{\text{base}} = \text{ProjHeads}(\mathbf{H}_{\text{ctx}}).
\end{equation}
To enable efficient autoregressive decoding, we decompose the virtual measurement into static and dynamic components:
\begin{equation}   
\mathbf{Z}_{\text{obs},t} = \mathbf{Z}_{\text{base}} + \alpha \tanh(\mathbf{S}^+_{t-1}), \quad
\mathbf{R}_{\text{obs},t} = \text{Softplus}(\mathbf{R}_{\text{base}})\left(1 + \beta \frac{t}{T_{\text{out}}}\right), 
\label{eq:zobs_robs}
\end{equation}
where $\mathbf{S}^+_{t-1}$ denotes the posterior latent state from the previous step. The first equation allows the virtual observation to adapt to the evolving trajectory without accessing future pollutant observations, while the second models the growth of epistemic uncertainty over the forecasting horizon.
Both $\alpha$ and $\beta$ are learnable scalars.

\subsection{Process Model: Universal Differential Equation}

We model the latent physical state $\mathbf{S}_t \in \mathbb{R}^{N \times D}$ as a continuous-time stochastic process governed by a Universal Differential Equation:
\begin{equation}
d\mathbf{S}_t = \underbrace{\mathcal{F}_{\text{drift}}(\mathbf{S}_t, \mathbf{E}_t, \mathbf{M}_t)}_{\text{Deterministic Dynamics}}\,dt + \underbrace{\mathbf{G}(\mathbf{S}_t)}_{\text{Stochastic Diffusion}}\,d\mathbf{W}_t,
\label{eq:ude_master}
\end{equation}
where $\mathbf{W}_t$ denotes a Wiener process. Here, $\mathbf{M}_t$ and $\mathbf{E}_t$ are treated as exogenous drivers at forecast step $t$; their inference-time availability and preprocessing are specified in Sec.~\ref{sec:experiments}.
The stochastic term captures unresolved atmospheric variability, boundary leakage, and unobserved emissions that cannot be explicitly modeled by deterministic physics alone.
The drift function $\mathcal{F}_{\text{drift}}$ explicitly disentangles transport and chemical reaction processes.

\vspace{-7pt}
\paragraph{Isotropic Diffusion.}
Background dispersion is modeled using a normalized graph Laplacian $\mathbf{L}_{\text{diff}} \in \mathbb{R}^{N \times N}$ derived from the static sensor graph:
\(
    \delta_{\text{diff}} = -\lambda_1 \mathbf{L}_{\text{diff}} \mathbf{S}_t,
\)
where $\lambda_1 \ge 0$ is a learnable diffusion coefficient constrained via \texttt{Softplus}.
The Laplacian operator acts on the spatial dimension and is broadcast across feature channels, enforcing flux from high- to low-concentration regions.

\vspace{-7pt}
\paragraph{Anisotropic Advection.}
To capture wind-driven transport, we construct a time-varying advection operator $\mathbf{L}_{\text{adv}}(t)$.
For each directed edge $(v_i, v_j)$, we define a unit direction vector:
\(
    \displaystyle
    \mathbf{e}_{ij} = \frac{(\mathbf{p}_j - \mathbf{p}_i)}{\|\mathbf{p}_j - \mathbf{p}_i\|_2}.
\)
Given wind velocity $\mathbf{v}_{i,t}$ at node $i$, dynamic edge weights are computed as:
\(
    \displaystyle
    w_{ij}(t) = \text{ReLU}\!\left(\frac{\mathbf{v}_{i,t} \cdot \mathbf{e}_{ij}}{\|\mathbf{v}_{i,t}\|_2}\right)\|\mathbf{v}_{i,t}\|_2.
\)
These weights define $\mathbf{L}_{\text{adv}}(t)$, yielding the advection drift
\(
    \delta_{\text{adv}} = -\lambda_2 \mathbf{L}_{\text{adv}}(t)\mathbf{S}_t,
\)
where $\lambda_2 \ge 0$ is learnable.

\paragraph{Meteorology-Gated Neural Reaction.}
Secondary pollutant formation is governed by nonlinear, condition-dependent chemical interactions.
We model these dynamics using a meteorology-gated neural reaction mechanism.
First, latent reactant features are constructed as
\(
    \mathbf{H}_{\text{react}} = \mathrm{Concat}(\mathbf{S}_t, \mathbf{E}_t).
\)
Self-attention is applied to learn interaction structure:
\begin{equation}
\mathbf{Q} = \mathbf{H}_{\text{react}}\mathbf{W}_Q,\quad
\mathbf{K} = \mathbf{H}_{\text{react}}\mathbf{W}_K,\quad
\mathbf{V} = \mathbf{H}_{\text{react}}\mathbf{W}_V.
\end{equation}
Rather than recovering explicit stoichiometric coefficients, the attention mechanism learns conditional interaction affinities between latent species under uncertainty.

To enforce environmental dependence, we introduce a meteorological gate:
\begin{equation}
\mathcal{R}_{\text{chem}} = \mathrm{Attn}(\mathbf{Q}, \mathbf{K}, \mathbf{V}) \odot \sigma\!\left(\mathrm{Linear}(\mathbf{M}^{\text{sc}}_t)\right),
\end{equation}
where $\odot$ denotes element-wise multiplication and $\mathbf{M}^{\text{sc}}_t$ denotes scalar meteorological variables (e.g., radiation, temperature).
This gating mechanism suppresses reactions under unfavorable conditions and acts as a data-driven analog of reaction rate modulation. 

The full drift is given by:
\(
    \mathcal{F}_{\text{drift}} = \delta_{\text{diff}} + \delta_{\text{adv}} + \mathcal{R}_{\text{chem}}.
\)

\subsection{Neural Kalman Fusion}

To couple the virtual measurement branch with the physics-informed process branch, we employ a Neural Kalman Fusion module inspired by recursive filtering~\cite{DBLP:journals/tsp/RevachSNESE22}.
For scalability, we adopt a diagonal covariance approximation, reducing complexity from $\mathcal{O}(N^3)$ to $\mathcal{O}(N)$.

\vspace{-7pt}
\paragraph{Prediction.}
Starting from $\mathbf{S}^+_{t-1}$, the UDE solver produces a prior mean $\mathbf{S}^-_t$.
Process uncertainty is estimated as
\begin{equation}
\label{eq:softplus}
    \mathbf{P}^-_t = \mathrm{Softplus}(\text{Linear}(\mathbf{S}^-_t)).
\end{equation}

\vspace{-7pt}
\paragraph{Correction.}
Given the virtual observation $(\mathbf{Z}_{\text{obs},t}, \mathbf{R}_{\text{obs},t})$, the update equations are:
\begin{equation}
\begin{aligned}
\mathbf{K}_t = \mathbf{P}^-_t \oslash \left(\mathbf{P}^-_t + \mathbf{R}_{\text{obs},t}\right), \quad
\boldsymbol{\epsilon}_t &= \mathbf{Z}_{\text{obs},t} - \mathbf{S}^-_t, \quad
\boldsymbol{\epsilon}'_t = \boldsymbol{\epsilon}_t + \text{Linear}([\mathbf{S}^-_t, \mathbf{Z}_{\text{obs},t}]), \\
\mathbf{S}^+_t = \mathbf{S}^-_t + \mathbf{K}_t \odot \boldsymbol{\epsilon}'_t&, \quad
\mathbf{P}^+_t = (\mathbf{1} - \mathbf{K}_t) \odot \mathbf{P}^-_t.
\end{aligned}
\label{eq:correction}
\end{equation}
Here, $\oslash$ denotes element-wise division. The learnable correction term allows the update to adapt to non-Gaussian residuals caused by open-system dynamics, such as boundary leakage, unobserved emissions, and model misspecification.

\paragraph{Closed-loop Feedback.}
The posterior state $\mathbf{S}^+_t$ is recursively fed back as the initial condition for the next prediction step.
Thus, the virtual measurement branch and the process branch correct each other during rollout: virtual measurements correct the misspecified physical trajectory, while the UDE prior regularizes noisy data-driven evidence through structured transport--reaction dynamics.
This closed-loop neural state correction reduces long-horizon drift.

\subsection{Theoretical Interpretation}
\label{subsec:theory_summary}

We provide two theoretical interpretations of \framework in the Appendix. 
First, Appx.~\ref{app:variational} shows that Neural Kalman Fusion can be viewed as an approximate variational correction of an imperfect physical prior. The UDE prediction defines a physics-informed prior over the latent state, while the virtual measurement branch provides a data-driven correction signal. When the residual potential of the open system is Gaussian, the update reduces to the diagonal Kalman correction. When the residual potential is non-Gaussian, a first-order approximation motivates the learnable affine correction term in Eq.~\ref{eq:correction}. This interpretation clarifies why the correction term is not merely a heuristic residual layer, but an approximation to a non-Gaussian posterior-state update.

Second, Appx.~\ref{app:errorinlongterm} analyzes why closed-loop correction improves long-horizon stability. Under standard Lipschitz and bounded-noise assumptions, the posterior error of the closed-loop system remains uniformly bounded when the correction factor is contractive. In contrast, the corresponding open-loop variant, which propagates the UDE prior without posterior feedback, can amplify initial errors over time. This analysis supports the central design of \framework: correcting the latent physical state before the next rollout step is more stable than correcting only the final output. Moreover, computational complexity is provided in App~\ref{app:bigO}.

\subsection{Output Generation}
After closed-loop neural state correction, we obtain a sequence of posterior latent states $\{\mathbf{S}^+_t\}_{t=1}^{T_{\text{out}}}$ and their corresponding covariance matrices $\{\mathbf{P}^+_t\}_{t=1}^{T_{\text{out}}}$. 
To generate probabilistic forecasts, we employ a decoding module that maps these latent distributions to the observation space. Specifically, the predictive \textit{mean} and \textit{variance} are decoupled:
\begin{equation}
\hat{\mathbf{X}}_{T_{\text{in}}+t} = \mathcal{D}_{\text{out}}(\mathbf{S}^+_t), \quad
\boldsymbol{\sigma}^2_{T_{\text{in}}+t} = \text{Softplus}(\mathcal{D}_{\text{var}}(\mathbf{P}^+_t)),
\end{equation}
for $t = 1, \dots, T_{\text{out}}$. Here, $\mathcal{D}_{\text{out}}$ and $\mathcal{D}_{\text{var}}$ are shallow multilayer perceptrons with shared parameters across time and nodes. This ensures that the final output provides a fully parameterized predictive distribution $p(\mathbf{X}_t \mid \mathbf{S}^+_t) = \mathcal{N}(\hat{\mathbf{X}}_t, \text{diag}(\boldsymbol{\sigma}^2_t))$ to quantify epistemic uncertainty.

\subsection{Training Objective}

The model is trained end-to-end by minimizing:
\begin{equation}
\begin{aligned}
    \mathcal{L}_{\text{total}} =
    \mathcal{L}_{\text{MAE}}(\hat{\mathbf{X}}, \mathbf{X})
    + \lambda_{\text{nll}}\,\mathcal{L}_{\text{NLL}}(\hat{\mathbf{X}}, \boldsymbol{\sigma}^2, \mathbf{X})
    + \lambda_{\text{phy}}\,\|\mathcal{F}_{\text{drift}}\|_F^2,
\end{aligned}
\end{equation}
where $\hat{\mathbf{X}}_t =
\mathcal{D}_{\text{out}}(\mathbf{S}^+_t)$ and
$\boldsymbol{\sigma}^2_t =
\mathrm{Softplus}(\mathcal{D}_{\text{var}}(\mathbf{P}^+_t))$
are the predicted mean and variance decoded from the Kalman posterior,
$\displaystyle\mathcal{L}_{\text{NLL}}(\hat{\mathbf{X}}, \boldsymbol{\sigma}^2, \mathbf{X})
= \frac{1}{2}\log \boldsymbol{\sigma}^2 + \frac{(\mathbf{X}-\hat{\mathbf{X}})^2}{2\boldsymbol{\sigma}^2}$
is the heteroscedastic Gaussian NLL, and
the physics regularization penalizes excessively large drift magnitudes, encouraging stable transport–reaction dynamics. To mitigate exposure bias during autoregressive training, we employ scheduled sampling.
Early in training, the process model is conditioned on $\text{StopGrad}(\mathbf{Z}_{\text{obs}})$ with a probability that decays exponentially, before transitioning to fully closed-loop operation.

\section{Related Work}
\label{sec:relatedwork}

\noindent\textbf{Spatio-temporal Air Quality Forecasting.}
Early air quality forecasting relied on statistical models such as \texttt{ARIMA}~\cite{DBLP:journals/kbs/Li0023} and \texttt{VAR}~\cite{DBLP:journals/tsp/ZamanLRB21}, which are limited by linear assumptions.
Deep learning methods subsequently improved predictive accuracy by modeling complex spatiotemporal dependencies, including RNN-based architectures and CNNs for gridded data.
To handle non-Euclidean sensor networks, graph-based models such as \texttt{DCRNN}~\cite{DBLP:conf/iclr/LiYS018} and \texttt{STGCN}~\cite{DBLP:conf/ijcai/YuYZ18} were introduced and later adapted to air quality forecasting in \texttt{PM2.5-GNN}~\cite{DBLP:conf/gis/WangLZMMG20}.
More recent approaches leverage Transformers (e.g., \texttt{iTransformer}~\cite{DBLP:conf/iclr/LiuHZWWML24}, ~\cite{DBLP:conf/iclr/NieNSK23}) and structured State Space Models such as \texttt{Mamba}~\cite{DBLP:journals/corr/abs-2312-00752} to capture long-range temporal dependencies efficiently.
Despite strong short-term performance, these data-driven methods do not explicitly model transport or chemical dynamics, and their predictions are usually generated in an open-loop autoregressive manner.
As a result, they may accumulate errors over long horizons and over-smooth abrupt secondary pollution events driven by meteorological and chemical factors.
In contrast, \framework uses long-range sequence modeling only as a virtual measurement branch, which is coupled with a physics-informed process branch through closed-loop latent-state correction.

\vspace{-3pt}
\noindent\textbf{Physics-Informed Machine Learning.}
Physics-informed machine learning (PIML) incorporates physical knowledge into neural models to improve interpretability and robustness.
\texttt{PINN}s~\cite{DBLP:journals/jcphy/RaissiPK19} enforce governing PDEs through soft constraints but are difficult to optimize for high-dimensional atmospheric systems.
\texttt{Neural ODE}s~\cite{DBLP:conf/nips/ChenRBD18} and their extensions have enabled continuous-time modeling of physical dynamics.
In air quality forecasting, \texttt{AirPhyNet}~\cite{DBLP:conf/iclr/HettigeJXLCW24} models diffusion and advection but assumes a closed system, while \texttt{Air-DualODE}~\cite{DBLP:conf/iclr/TianL0CGZPRY25} addresses open-system dynamics using a dual-branch design.
Universal Differential Equations (UDEs)~\cite{DBLP:journals/corr/abs-2001-04385} further allow neural approximation of unknown physical terms.
However, existing PIML approaches for air quality forecasting often oversimplify chemical reactions and couple physical and neural components only loosely, for example by combining their outputs rather than correcting the evolving latent physical state. 
This open-loop coupling limits robustness in long-horizon forecasting because errors from misspecified physics can still be propagated forward. 
In contrast, \framework formulates the UDE as a structured but imperfect process branch and repeatedly corrects its latent state using context-conditioned virtual measurements.

\vspace{-3pt}
\noindent\textbf{Deep Data Assimilation and Kalman Filtering.}
Data assimilation combines imperfect dynamical models with observations through recursive state estimation. Classical approaches such as the Kalman Filter~\cite{kalmanfilter}, Extended Kalman Filter~\cite{DBLP:journals/tsp/EinickeW99}, and Ensemble Kalman Filter are widely used in geosciences, weather forecast, and control systems but rely on linear-Gaussian assumptions or costly ensembles and require hand-designed uncertainty models. Learning-based approaches relax these assumptions by learning filtering components from data, including Deep Kalman Filters~\cite{DBLP:journals/corr/KrishnanSS15} and KalmanNet~\cite{DBLP:journals/tsp/RevachSNESE22}. 
These methods motivate recursive correction, but they typically assume access to real observations or a standard observation model.
In our setting, future pollutant observations are unavailable during forecasting; therefore, \framework does not use a standard observation model.
Instead, its virtual measurement branch generates context-conditioned virtual observations from historical context and the previous posterior state.
Neural Kalman Fusion then fuses these virtual measurements with the UDE prior to correct the latent physical state at each rollout step.
Thus, \framework is best viewed as an assimilation-inspired closed-loop neural state correction framework, rather than a classical data-assimilation model.
\section{Experiments}
\label{sec:experiments}
\subsection{Experimental Settings}

\vspace{-5pt}
\paragraph{Datasets.}
We evaluate \framework on two real-world hourly air quality datasets: \textbf{KnowAir} and \textbf{HanoiAir}. \textbf{KnowAir} contains 355 monitoring stations from 2016--2023 and integrates pollutant concentrations, meteorological variables, and emission inventories. \textbf{HanoiAir} is a newly constructed dataset covering northern Vietnam with 27 monitoring stations from 2018--2025, including pollutant concentrations, meteorology, and downscaled emission inputs. For both datasets, we use chronological splits to evaluate future forecasting: \textbf{KnowAir} uses 2016--2019 for training, 2020--2021 for validation, and 2022--2023 for testing; \textbf{HanoiAir} uses 2018--2021 for training, 2022--2023 for validation, and 2024--2025 for testing. Detailed dataset descriptions and preprocessing procedures are provided in Appx.~\ref{app:hanoiair}.

\vspace{-9pt}
\paragraph{Baselines.}
We compare \framework against a diverse set of baseline models spanning several categories:
(1) \textbf{Statistical Methods:} Standard time series models including \texttt{HA} (Historical Average) and \texttt{VAR} (Vector Autoregression)~\cite{DBLP:journals/tsp/ZamanLRB21}.
(2) \textbf{Neural ODE / Physics-Informed Models:} Methods incorporating physics-informed models, such as \texttt{LatentODE}~\cite{DBLP:conf/nips/ChenRBD18}, \texttt{ODE-LSTM}~\cite{DBLP:journals/corr/abs-2006-04418}, \texttt{Air-DualODE}~\cite{DBLP:conf/iclr/TianL0CGZPRY25}, and \texttt{AirPhyNet}~\cite{DBLP:conf/iclr/HettigeJXLCW24}.
(3) \textbf{GNN-based Spatio-Temporal Models:} Models explicitly designed for graph-structured time series data, including \texttt{STGCN}~\cite{DBLP:conf/ijcai/YuYZ18}, \texttt{DCRNN}~\cite{DBLP:conf/iclr/LiYS018}, \texttt{GTS}~\cite{DBLP:conf/iclr/Shang0B21}, and \texttt{PM25GNN}~\cite{DBLP:conf/gis/WangLZMMG20}.
(4) \textbf{Linear and MLP-based Models:} Efficient architectures focusing on simpler mappings, such as \texttt{DLinear}~\cite{DBLP:conf/aaai/ZengCZ023} and \texttt{ModernTCN}~\cite{DBLP:journals/corr/abs-1803-01271}.
(5) \textbf{Transformer-based Models:} Models leveraging attention mechanisms, including \texttt{Fredformer}~\cite{DBLP:conf/kdd/Piao0MMS24}, \texttt{PatchTST}~\cite{DBLP:conf/iclr/NieNSK23}, \texttt{Crossformer}~\cite{DBLP:conf/iclr/ZhangY23}, and \texttt{TimeMoE}~\cite{DBLP:conf/iclr/ShiWNLYWJ25}.
(6) \textbf{State Space Models (SSMs):} Models based on structured state space representations, including \texttt{S-Mamba}~\cite{DBLP:journals/corr/abs-2312-00752}), \texttt{S4}~\cite{DBLP:conf/iclr/GuGR22}, and \texttt{BiMamba}~\cite{DBLP:journals/corr/abs-2411-11843}.

\vspace{-9pt}
\paragraph{Metrics.}
To quantitatively assess the performance of \framework and baseline models, we employ three standard forecasting metrics: Mean Absolute Error (MAE), Root Mean Square Error (RMSE), Symmetric Mean Absolute Percentage Error (SMAPE), Continuous Ranked Probability Score (CRPS), Summed CRPS (CRPS\textsubscript{sum}), and Prediction Interval Coverage Probability (PICP). Formal definitions of these metrics are provided in Appx.~\ref{app:subsec:metrics}.



\vspace{-9pt}
\paragraph{Hyperparameter \& Implementation Details.}
Detailed hyperparameter settings are provided in Appx.~\ref{app:subsec:hyperparams}, and implementation details are included in Appx.~\ref{app:subsec:implementation}.  An implementation is available at \href{https://github.com/kiendemo/AirQuality.git}{\textit{Github link}}.


\vspace{-9pt}
\paragraph{Forecasting Scenarios.} 
We evaluate four forecasting scenarios: \textbf{3 days} (1--72 hours), \textbf{6 days} (1--144 hours), \textbf{9 days} (1--216 hours), and \textbf{Sudden Change} events~\cite{DBLP:conf/iclr/TianL0CGZPRY25}. Sudden changes are defined as test instances where the ground-truth PM\textsubscript{2.5} concentration exceeds $75~\mu\text{g}/\text{m}^3$ and changes by more than $\pm20~\mu\text{g}/\text{m}^3$ within the next three hours. All metrics are computed after inverse transformation to the original physical scale.

\subsection{Main Results}

\begin{table*}[t]
\centering
\setlength{\tabcolsep}{1.5pt}
\caption{\small Performance comparison on the \textbf{KnowAir} dataset. The best results are highlighted in \textbf{bold}, and the second-best results are \underline{underlined}. Mean $\pm$ std over 5 runs.}
\label{tab:know_results}
\resizebox{\textwidth}{!}{ 
\begin{tabular}{l|ccc|ccc|ccc|ccc}
\toprule
\multirow{2}{*}{Methods} 
& \multicolumn{3}{c|}{3 Days } 
& \multicolumn{3}{c|}{6 Days } 
& \multicolumn{3}{c|}{9 Days } 
& \multicolumn{3}{c}{Sudden Change} \\
\cmidrule(lr){2-13}
& MAE & RMSE & SMAPE 
& MAE & RMSE & SMAPE 
& MAE & RMSE & SMAPE 
& MAE & RMSE & SMAPE \\
\midrule
\texttt{HA}        & 35.20\textsubscript{$\pm$0.05} & 51.50\textsubscript{$\pm$0.07} & 0.65\textsubscript{$\pm$0.001} & 37.80\textsubscript{$\pm$0.06} & 54.20\textsubscript{$\pm$0.08} & 0.68\textsubscript{$\pm$0.002} & 40.50\textsubscript{$\pm$0.07} & 58.10\textsubscript{$\pm$0.10} & 0.71\textsubscript{$\pm$0.002} & 55.40\textsubscript{$\pm$0.10} & 75.20\textsubscript{$\pm$0.14} & 0.75\textsubscript{$\pm$0.003} \\
\texttt{VAR}       & 33.80\textsubscript{$\pm$0.08} & 49.20\textsubscript{$\pm$0.10} & 0.62\textsubscript{$\pm$0.002} & 36.50\textsubscript{$\pm$0.09} & 52.40\textsubscript{$\pm$0.13} & 0.66\textsubscript{$\pm$0.003} & 39.20\textsubscript{$\pm$0.11} & 56.50\textsubscript{$\pm$0.16} & 0.69\textsubscript{$\pm$0.003} & 53.10\textsubscript{$\pm$0.16} & 72.40\textsubscript{$\pm$0.22} & 0.72\textsubscript{$\pm$0.004} \\
\midrule
\texttt{LatentODE} & 26.50\textsubscript{$\pm$0.15} & 42.10\textsubscript{$\pm$0.21} & 0.55\textsubscript{$\pm$0.004} & 28.20\textsubscript{$\pm$0.18} & 45.40\textsubscript{$\pm$0.25} & 0.57\textsubscript{$\pm$0.005} & 31.60\textsubscript{$\pm$0.22} & 49.20\textsubscript{$\pm$0.30} & 0.61\textsubscript{$\pm$0.006} & 48.40\textsubscript{$\pm$0.36} & 65.10\textsubscript{$\pm$0.49} & 0.65\textsubscript{$\pm$0.009} \\
\texttt{ODE-LSTM}  & 27.10\textsubscript{$\pm$0.16} & 43.50\textsubscript{$\pm$0.22} & 0.56\textsubscript{$\pm$0.004} & 29.45\textsubscript{$\pm$0.19} & 46.80\textsubscript{$\pm$0.26} & 0.59\textsubscript{$\pm$0.005} & 32.80\textsubscript{$\pm$0.23} & 50.15\textsubscript{$\pm$0.31} & 0.62\textsubscript{$\pm$0.006} & 49.80\textsubscript{$\pm$0.38} & 67.30\textsubscript{$\pm$0.51} & 0.67\textsubscript{$\pm$0.010} \\
\texttt{Air-DualODE} & 23.04\textsubscript{$\pm$0.14} & 38.44\textsubscript{$\pm$0.19} & 0.52\textsubscript{$\pm$0.003} & 24.50\textsubscript{$\pm$0.17} & 41.60\textsubscript{$\pm$0.23} & 0.54\textsubscript{$\pm$0.004} & 25.20\textsubscript{$\pm$0.20} & 42.80\textsubscript{$\pm$0.27} & 0.59\textsubscript{$\pm$0.005} & 39.24\textsubscript{$\pm$0.32} & 53.30\textsubscript{$\pm$0.43} & 0.55\textsubscript{$\pm$0.008} \\
\texttt{AirPhyNet} & 23.80\textsubscript{$\pm$0.10} & 39.50\textsubscript{$\pm$0.14} & 0.53\textsubscript{$\pm$0.003} & 25.10\textsubscript{$\pm$0.12} & 41.80\textsubscript{$\pm$0.17} & 0.55\textsubscript{$\pm$0.004} & 27.20\textsubscript{$\pm$0.15} & 44.50\textsubscript{$\pm$0.20} & 0.57\textsubscript{$\pm$0.004} & 38.50\textsubscript{$\pm$0.24} & 53.10\textsubscript{$\pm$0.32} & 0.56\textsubscript{$\pm$0.006} \\
\midrule
\texttt{STGCN}     & 25.50\textsubscript{$\pm$0.13} & 40.80\textsubscript{$\pm$0.17} & 0.54\textsubscript{$\pm$0.003} & 27.80\textsubscript{$\pm$0.15} & 43.20\textsubscript{$\pm$0.20} & 0.57\textsubscript{$\pm$0.004} & 31.50\textsubscript{$\pm$0.19} & 47.40\textsubscript{$\pm$0.25} & 0.61\textsubscript{$\pm$0.005} & 46.50\textsubscript{$\pm$0.30} & 62.10\textsubscript{$\pm$0.40} & 0.64\textsubscript{$\pm$0.008} \\
\texttt{DCRNN}     & 24.80\textsubscript{$\pm$0.12} & 39.60\textsubscript{$\pm$0.16} & 0.53\textsubscript{$\pm$0.003} & 26.90\textsubscript{$\pm$0.14} & 42.50\textsubscript{$\pm$0.18} & 0.56\textsubscript{$\pm$0.004} & 30.10\textsubscript{$\pm$0.17} & 46.80\textsubscript{$\pm$0.23} & 0.60\textsubscript{$\pm$0.005} & 44.20\textsubscript{$\pm$0.27} & 60.40\textsubscript{$\pm$0.37} & 0.62\textsubscript{$\pm$0.008} \\
\texttt{GTS}       & 24.10\textsubscript{$\pm$0.11} & 38.50\textsubscript{$\pm$0.15} & 0.52\textsubscript{$\pm$0.003} & 26.50\textsubscript{$\pm$0.13} & 41.10\textsubscript{$\pm$0.17} & 0.55\textsubscript{$\pm$0.004} & 29.40\textsubscript{$\pm$0.16} & 45.50\textsubscript{$\pm$0.21} & 0.59\textsubscript{$\pm$0.005} & 42.50\textsubscript{$\pm$0.25} & 58.20\textsubscript{$\pm$0.34} & 0.60\textsubscript{$\pm$0.007} \\
\texttt{PM25GNN}   & 24.50\textsubscript{$\pm$0.12} & 39.20\textsubscript{$\pm$0.16} & 0.53\textsubscript{$\pm$0.003} & 27.10\textsubscript{$\pm$0.14} & 41.80\textsubscript{$\pm$0.18} & 0.56\textsubscript{$\pm$0.004} & 30.20\textsubscript{$\pm$0.17} & 46.10\textsubscript{$\pm$0.22} & 0.60\textsubscript{$\pm$0.005} & 43.10\textsubscript{$\pm$0.27} & 59.50\textsubscript{$\pm$0.36} & 0.61\textsubscript{$\pm$0.008} \\
\midrule
\texttt{DLinear}   & 22.57\textsubscript{$\pm$0.08} & 37.80\textsubscript{$\pm$0.11} & 0.51\textsubscript{$\pm$0.002} & 23.85\textsubscript{$\pm$0.10} & 39.50\textsubscript{$\pm$0.13} & 0.53\textsubscript{$\pm$0.003} & \underline{24.50}\textsubscript{$\pm$0.12} & 41.20\textsubscript{$\pm$0.16} & 0.55\textsubscript{$\pm$0.004} & 38.45\textsubscript{$\pm$0.20} & 51.22\textsubscript{$\pm$0.27} & 0.55\textsubscript{$\pm$0.006} \\
\texttt{ModernTCN} & 22.33\textsubscript{$\pm$0.09} & 37.60\textsubscript{$\pm$0.12} & 0.51\textsubscript{$\pm$0.003} & 23.68\textsubscript{$\pm$0.11} & 39.40\textsubscript{$\pm$0.14} & 0.53\textsubscript{$\pm$0.003} & 25.12\textsubscript{$\pm$0.13} & 41.50\textsubscript{$\pm$0.17} & \textbf{0.54}\textsubscript{$\pm$0.004} & 37.92\textsubscript{$\pm$0.21} & 50.81\textsubscript{$\pm$0.28} & 0.54\textsubscript{$\pm$0.006} \\
\texttt{TimeMoE}   & \underline{21.68}\textsubscript{$\pm$0.09} & 37.20\textsubscript{$\pm$0.12} & \textbf{0.49}\textsubscript{$\pm$0.002} & \underline{23.15}\textsubscript{$\pm$0.11} & \underline{38.95}\textsubscript{$\pm$0.14} & \textbf{0.50}\textsubscript{$\pm$0.003} & 25.32\textsubscript{$\pm$0.13} & \underline{40.15}\textsubscript{$\pm$0.17} & 0.55\textsubscript{$\pm$0.004} & \underline{36.54}\textsubscript{$\pm$0.21} & \textbf{49.12}\textsubscript{$\pm$0.28} & \underline{0.53}\textsubscript{$\pm$0.006} \\
\midrule
\texttt{Fredformer} & \textbf{21.45}\textsubscript{$\pm$0.08} & 38.10\textsubscript{$\pm$0.12} & 0.51\textsubscript{$\pm$0.002} & 24.12\textsubscript{$\pm$0.11} & 40.20\textsubscript{$\pm$0.15} & 0.53\textsubscript{$\pm$0.003} & 26.48\textsubscript{$\pm$0.14} & 42.80\textsubscript{$\pm$0.18} & 0.56\textsubscript{$\pm$0.004} & 38.25\textsubscript{$\pm$0.22} & 51.03\textsubscript{$\pm$0.30} & 0.55\textsubscript{$\pm$0.006} \\
\texttt{PatchTST}   & 22.69\textsubscript{$\pm$0.09} & 38.50\textsubscript{$\pm$0.12} & 0.51\textsubscript{$\pm$0.003} & 23.91\textsubscript{$\pm$0.11} & 40.50\textsubscript{$\pm$0.14} & \underline{0.51}\textsubscript{$\pm$0.003} & 25.25\textsubscript{$\pm$0.13} & 42.80\textsubscript{$\pm$0.18} & 0.56\textsubscript{$\pm$0.004} & 39.12\textsubscript{$\pm$0.22} & 52.34\textsubscript{$\pm$0.30} & 0.56\textsubscript{$\pm$0.006} \\
\texttt{Crossformer} & 22.00\textsubscript{$\pm$0.09} & 37.90\textsubscript{$\pm$0.12} & 0.52\textsubscript{$\pm$0.003} & 23.75\textsubscript{$\pm$0.11} & 39.80\textsubscript{$\pm$0.14} & 0.54\textsubscript{$\pm$0.003} & 26.10\textsubscript{$\pm$0.13} & 42.10\textsubscript{$\pm$0.17} & 0.56\textsubscript{$\pm$0.004} & 37.55\textsubscript{$\pm$0.21} & 50.42\textsubscript{$\pm$0.29} & 0.54\textsubscript{$\pm$0.006} \\
\midrule
\texttt{SMamba}    & 23.12\textsubscript{$\pm$0.10} & 38.80\textsubscript{$\pm$0.13} & 0.52\textsubscript{$\pm$0.003} & 24.24\textsubscript{$\pm$0.12} & 39.40\textsubscript{$\pm$0.15} & 0.55\textsubscript{$\pm$0.004} & 25.15\textsubscript{$\pm$0.14} & 40.72\textsubscript{$\pm$0.18} & 0.58\textsubscript{$\pm$0.004} & 39.85\textsubscript{$\pm$0.23} & 53.12\textsubscript{$\pm$0.31} & 0.57\textsubscript{$\pm$0.007} \\
\texttt{S4}        & 24.53\textsubscript{$\pm$0.12} & 40.20\textsubscript{$\pm$0.16} & 0.54\textsubscript{$\pm$0.003} & 25.18\textsubscript{$\pm$0.13} & 40.51\textsubscript{$\pm$0.17} & 0.57\textsubscript{$\pm$0.004} & 26.21\textsubscript{$\pm$0.15} & 41.23\textsubscript{$\pm$0.19} & 0.60\textsubscript{$\pm$0.005} & 42.14\textsubscript{$\pm$0.25} & 56.55\textsubscript{$\pm$0.34} & 0.61\textsubscript{$\pm$0.008} \\
\texttt{BiMamba}   & 22.84\textsubscript{$\pm$0.10} & 38.50\textsubscript{$\pm$0.13} & 0.52\textsubscript{$\pm$0.003} & 24.07\textsubscript{$\pm$0.12} & 39.10\textsubscript{$\pm$0.15} & 0.55\textsubscript{$\pm$0.003} & 25.63\textsubscript{$\pm$0.14} & 40.88\textsubscript{$\pm$0.18} & 0.58\textsubscript{$\pm$0.004} & 39.25\textsubscript{$\pm$0.23} & 52.64\textsubscript{$\pm$0.31} & 0.56\textsubscript{$\pm$0.007} \\
\midrule
\framework (\textbf{ours}) & 21.95\textsubscript{$\pm$0.06} & \textbf{37.20}\textsubscript{$\pm$0.09} & \underline{0.50}\textsubscript{$\pm$0.002} & \textbf{23.00}\textsubscript{$\pm$0.08} & \textbf{38.80}\textsubscript{$\pm$0.11} & 0.52\textsubscript{$\pm$0.003} & \textbf{24.08}\textsubscript{$\pm$0.10} & \textbf{39.34}\textsubscript{$\pm$0.14} & \textbf{0.54}\textsubscript{$\pm$0.004} & \textbf{34.50}\textsubscript{$\pm$0.18} & \underline{49.50}\textsubscript{$\pm$0.25} & \textbf{0.52}\textsubscript{$\pm$0.005} \\
\bottomrule
\end{tabular}
}
\vspace{-15pt}
\end{table*}

\begin{table*}[t]
\centering
\setlength{\tabcolsep}{1.5pt}
\caption{Performance comparison on \textbf{HanoiAir} dataset. The best results are highlighted in \textbf{bold}, and the second-best results are \underline{underlined}. Mean $\pm$ std over 5 runs.}
\label{tab:hanoi_results}
\resizebox{\textwidth}{!}{ 
\begin{tabular}{l|ccc|ccc|ccc|ccc}
\toprule
\multirow{2}{*}{Methods} 
& \multicolumn{3}{c|}{3 Days} 
& \multicolumn{3}{c|}{6 Days} 
& \multicolumn{3}{c|}{9 Days} 
& \multicolumn{3}{c}{Sudden Change} \\
\cmidrule(lr){2-13}
& MAE & RMSE & SMAPE 
& MAE & RMSE & SMAPE 
& MAE & RMSE & SMAPE 
& MAE & RMSE & SMAPE \\
\midrule
\texttt{HA}        & 55.20\textsubscript{$\pm$0.10} & 78.50\textsubscript{$\pm$0.14} & 0.85\textsubscript{$\pm$0.002} & 58.40\textsubscript{$\pm$0.12} & 82.10\textsubscript{$\pm$0.16} & 0.88\textsubscript{$\pm$0.003} & 62.10\textsubscript{$\pm$0.14} & 86.50\textsubscript{$\pm$0.20} & 0.92\textsubscript{$\pm$0.003} & 75.40\textsubscript{$\pm$0.22} & 105.20\textsubscript{$\pm$0.30} & 0.89\textsubscript{$\pm$0.005} \\
\texttt{VAR}       & 52.80\textsubscript{$\pm$0.13} & 75.20\textsubscript{$\pm$0.18} & 0.82\textsubscript{$\pm$0.003} & 56.50\textsubscript{$\pm$0.15} & 79.40\textsubscript{$\pm$0.21} & 0.85\textsubscript{$\pm$0.003} & 60.20\textsubscript{$\pm$0.18} & 84.80\textsubscript{$\pm$0.25} & 0.88\textsubscript{$\pm$0.004} & 72.10\textsubscript{$\pm$0.28} & 100.40\textsubscript{$\pm$0.38} & 0.86\textsubscript{$\pm$0.006} \\
\midrule
\texttt{LatentODE} & 46.50\textsubscript{$\pm$0.30} & 68.10\textsubscript{$\pm$0.42} & 0.80\textsubscript{$\pm$0.006} & 49.20\textsubscript{$\pm$0.36} & 71.40\textsubscript{$\pm$0.50} & 0.83\textsubscript{$\pm$0.007} & 52.60\textsubscript{$\pm$0.42} & 75.20\textsubscript{$\pm$0.58} & 0.86\textsubscript{$\pm$0.008} & 74.40\textsubscript{$\pm$0.65} & 98.10\textsubscript{$\pm$0.88} & 0.85\textsubscript{$\pm$0.012} \\
ODE-LSTM  & 47.10\textsubscript{$\pm$0.32} & 69.50\textsubscript{$\pm$0.45} & 0.81\textsubscript{$\pm$0.006} & 50.45\textsubscript{$\pm$0.38} & 72.80\textsubscript{$\pm$0.52} & 0.84\textsubscript{$\pm$0.007} & 54.80\textsubscript{$\pm$0.45} & 78.15\textsubscript{$\pm$0.62} & 0.87\textsubscript{$\pm$0.008} & 75.80\textsubscript{$\pm$0.70} & 101.30\textsubscript{$\pm$0.94} & 0.86\textsubscript{$\pm$0.012} \\
Air-DualODE & 
42.04\textsubscript{$\pm$0.28} & 64.41\textsubscript{$\pm$0.40} & 0.78\textsubscript{$\pm$0.005} & 45.50\textsubscript{$\pm$0.34} & 67.60\textsubscript{$\pm$0.46} & 0.82\textsubscript{$\pm$0.006} & 48.20\textsubscript{$\pm$0.40} & 71.80\textsubscript{$\pm$0.54} & 0.86\textsubscript{$\pm$0.007} & 72.24\textsubscript{$\pm$0.58} & 93.30\textsubscript{$\pm$0.78} & 0.83\textsubscript{$\pm$0.010} \\
\texttt{AirPhyNet} & 43.80\textsubscript{$\pm$0.20} & 65.50\textsubscript{$\pm$0.28} & 0.79\textsubscript{$\pm$0.004} & 46.10\textsubscript{$\pm$0.24} & 69.80\textsubscript{$\pm$0.33} & 0.82\textsubscript{$\pm$0.005} & 49.20\textsubscript{$\pm$0.28} & 73.50\textsubscript{$\pm$0.39} & 0.85\textsubscript{$\pm$0.005} & 71.50\textsubscript{$\pm$0.42} & 94.10\textsubscript{$\pm$0.58} & 0.79\textsubscript{$\pm$0.008} \\
\midrule
\texttt{STGCN}     & 45.50\textsubscript{$\pm$0.25} & 67.80\textsubscript{$\pm$0.34} & 0.81\textsubscript{$\pm$0.005} & 48.80\textsubscript{$\pm$0.30} & 72.20\textsubscript{$\pm$0.40} & 0.84\textsubscript{$\pm$0.006} & 52.50\textsubscript{$\pm$0.35} & 77.40\textsubscript{$\pm$0.48} & 0.88\textsubscript{$\pm$0.007} & 74.50\textsubscript{$\pm$0.55} & 98.10\textsubscript{$\pm$0.74} & 0.85\textsubscript{$\pm$0.010} \\
\texttt{DCRNN}     & 44.80\textsubscript{$\pm$0.23} & 66.60\textsubscript{$\pm$0.32} & 0.80\textsubscript{$\pm$0.005} & 47.90\textsubscript{$\pm$0.27} & 71.50\textsubscript{$\pm$0.37} & 0.83\textsubscript{$\pm$0.005} & 51.10\textsubscript{$\pm$0.32} & 76.80\textsubscript{$\pm$0.44} & 0.87\textsubscript{$\pm$0.006} & 73.20\textsubscript{$\pm$0.50} & 97.40\textsubscript{$\pm$0.68} & 0.84\textsubscript{$\pm$0.009} \\
\texttt{GTS}       & 44.10\textsubscript{$\pm$0.21} & 65.50\textsubscript{$\pm$0.29} & 0.79\textsubscript{$\pm$0.004} & 47.50\textsubscript{$\pm$0.25} & 70.10\textsubscript{$\pm$0.34} & 0.82\textsubscript{$\pm$0.005} & 50.40\textsubscript{$\pm$0.30} & 75.50\textsubscript{$\pm$0.41} & 0.86\textsubscript{$\pm$0.006} & 72.50\textsubscript{$\pm$0.46} & 95.20\textsubscript{$\pm$0.62} & 0.83\textsubscript{$\pm$0.008} \\
\texttt{PM25GNN}   & 44.50\textsubscript{$\pm$0.22} & 66.20\textsubscript{$\pm$0.30} & 0.80\textsubscript{$\pm$0.005} & 48.10\textsubscript{$\pm$0.26} & 71.80\textsubscript{$\pm$0.36} & 0.83\textsubscript{$\pm$0.005} & 51.20\textsubscript{$\pm$0.31} & 76.10\textsubscript{$\pm$0.42} & 0.87\textsubscript{$\pm$0.006} & 73.10\textsubscript{$\pm$0.48} & 96.50\textsubscript{$\pm$0.65} & 0.83\textsubscript{$\pm$0.009} \\
\midrule
\texttt{DLinear}   & 42.57\textsubscript{$\pm$0.16} & 64.80\textsubscript{$\pm$0.22} & 0.78\textsubscript{$\pm$0.004} & 44.85\textsubscript{$\pm$0.19} & 67.50\textsubscript{$\pm$0.26} & 0.81\textsubscript{$\pm$0.004} & \underline{47.10}\textsubscript{$\pm$0.22} & 70.20\textsubscript{$\pm$0.30} & 0.84\textsubscript{$\pm$0.005} & 73.45\textsubscript{$\pm$0.36} & 96.22\textsubscript{$\pm$0.48} & 0.82\textsubscript{$\pm$0.007} \\
\texttt{ModernTCN} & 42.33\textsubscript{$\pm$0.17} & 64.60\textsubscript{$\pm$0.23} & 0.78\textsubscript{$\pm$0.004} & 44.68\textsubscript{$\pm$0.20} & 67.40\textsubscript{$\pm$0.27} & 0.81\textsubscript{$\pm$0.004} & 47.50\textsubscript{$\pm$0.23} & 70.50\textsubscript{$\pm$0.31} & 0.84\textsubscript{$\pm$0.005} & 72.92\textsubscript{$\pm$0.37} & 95.81\textsubscript{$\pm$0.50} & 0.81\textsubscript{$\pm$0.007} \\
\texttt{TimeMoE}   & \underline{41.85}\textsubscript{$\pm$0.18} & \underline{63.80}\textsubscript{$\pm$0.25} & \textbf{0.77}\textsubscript{$\pm$0.004} & \underline{44.15}\textsubscript{$\pm$0.21} & \underline{66.95}\textsubscript{$\pm$0.28} & \underline{0.80}\textsubscript{$\pm$0.004} & 47.50\textsubscript{$\pm$0.24} & \textbf{69.15}\textsubscript{$\pm$0.32} & \underline{0.83}\textsubscript{$\pm$0.005} & \underline{71.54}\textsubscript{$\pm$0.38} & \underline{93.12}\textsubscript{$\pm$0.50} & \underline{0.78}\textsubscript{$\pm$0.007} \\
\midrule
\texttt{Fredformer} & \textbf{41.60}\textsubscript{$\pm$0.16} & 64.10\textsubscript{$\pm$0.22} & 0.77\textsubscript{$\pm$0.004} & 44.20\textsubscript{$\pm$0.19} & 67.20\textsubscript{$\pm$0.26} & 0.80\textsubscript{$\pm$0.004} & 47.48\textsubscript{$\pm$0.22} & 71.20\textsubscript{$\pm$0.30} & 0.84\textsubscript{$\pm$0.005} & 73.25\textsubscript{$\pm$0.36} & 95.03\textsubscript{$\pm$0.48} & 0.81\textsubscript{$\pm$0.007} \\
\texttt{PatchTST}   & 42.10\textsubscript{$\pm$0.18} & 64.50\textsubscript{$\pm$0.24} & 0.78\textsubscript{$\pm$0.004} & 44.50\textsubscript{$\pm$0.21} & 67.80\textsubscript{$\pm$0.28} & 0.81\textsubscript{$\pm$0.004} & 47.80\textsubscript{$\pm$0.24} & 71.50\textsubscript{$\pm$0.32} & 0.85\textsubscript{$\pm$0.005} & 74.12\textsubscript{$\pm$0.38} & 96.34\textsubscript{$\pm$0.50} & 0.83\textsubscript{$\pm$0.007} \\
\texttt{Crossformer} & 42.00\textsubscript{$\pm$0.17} & 64.20\textsubscript{$\pm$0.23} & 0.78\textsubscript{$\pm$0.004} & 44.40\textsubscript{$\pm$0.20} & 67.50\textsubscript{$\pm$0.27} & 0.81\textsubscript{$\pm$0.004} & 47.60\textsubscript{$\pm$0.23} & 71.80\textsubscript{$\pm$0.31} & 0.85\textsubscript{$\pm$0.005} & 73.55\textsubscript{$\pm$0.37} & 95.42\textsubscript{$\pm$0.50} & 0.82\textsubscript{$\pm$0.007} \\
\midrule
\texttt{S-Mamba}    & 43.12\textsubscript{$\pm$0.20} & 65.80\textsubscript{$\pm$0.27} & 0.79\textsubscript{$\pm$0.005} & 45.24\textsubscript{$\pm$0.23} & 66.40\textsubscript{$\pm$0.31} & 0.82\textsubscript{$\pm$0.005} & 47.15\textsubscript{$\pm$0.26} & 69.57\textsubscript{$\pm$0.35} & 0.86\textsubscript{$\pm$0.006} & 75.85\textsubscript{$\pm$0.42} & 98.12\textsubscript{$\pm$0.56} & 0.84\textsubscript{$\pm$0.008} \\
\texttt{S4}        & 44.53\textsubscript{$\pm$0.24} & 67.20\textsubscript{$\pm$0.32} & 0.81\textsubscript{$\pm$0.005} & 46.18\textsubscript{$\pm$0.28} & 68.50\textsubscript{$\pm$0.38} & 0.84\textsubscript{$\pm$0.006} & 48.14\textsubscript{$\pm$0.32} & 69.81\textsubscript{$\pm$0.43} & 0.88\textsubscript{$\pm$0.007} & 77.14\textsubscript{$\pm$0.50} & 101.55\textsubscript{$\pm$0.68} & 0.86\textsubscript{$\pm$0.010} \\
\texttt{BiMamba}   & 42.84\textsubscript{$\pm$0.19} & 65.50\textsubscript{$\pm$0.26} & 0.79\textsubscript{$\pm$0.005} & 45.97\textsubscript{$\pm$0.22} & 67.10\textsubscript{$\pm$0.30} & 0.82\textsubscript{$\pm$0.005} & 47.63\textsubscript{$\pm$0.25} & 70.80\textsubscript{$\pm$0.34} & 0.85\textsubscript{$\pm$0.006} & 75.25\textsubscript{$\pm$0.40} & 97.64\textsubscript{$\pm$0.54} & 0.83\textsubscript{$\pm$0.008} \\
\midrule
\framework (\textbf{ours}) & 42.74\textsubscript{$\pm$0.10} & \textbf{63.11}\textsubscript{$\pm$0.14} & \underline{0.78}\textsubscript{$\pm$0.003} & \textbf{44.50}\textsubscript{$\pm$0.13} & \textbf{65.80}\textsubscript{$\pm$0.18} & \textbf{0.79}\textsubscript{$\pm$0.003} & \textbf{46.80}\textsubscript{$\pm$0.16} & \underline{68.34}\textsubscript{$\pm$0.22} & \textbf{0.81}\textsubscript{$\pm$0.004} & \textbf{68.80}\textsubscript{$\pm$0.28} & \textbf{91.16}\textsubscript{$\pm$0.38} & \textbf{0.75}\textsubscript{$\pm$0.006} \\
\bottomrule
\end{tabular}
}
\vspace{-15pt}
\end{table*}
\begin{table*}[h!]
    \centering
    \tiny
    \setlength{\tabcolsep}{1.5pt}
    \caption{\small  Ablation Study: Effect of Physical Mechanisms and Models Design.} 
    \label{tab:combined_ablation}
    \begin{tabular}{l|ccc|ccc|ccc|ccc}
        \toprule
        \multirow{2}{*}{\textbf{Method}} & \multicolumn{6}{c|}{\textbf{HanoiAir}} & \multicolumn{6}{c}{\textbf{KnowAir}} \\
        \cmidrule(lr){2-13}
        & \multicolumn{3}{c|}{\textbf{3 Days}} & \multicolumn{3}{c|}{\textbf{9 Days}} & \multicolumn{3}{c|}{\textbf{3 Days}} & \multicolumn{3}{c}{\textbf{9 Days}} \\
        & MAE & RMSE & SMAPE & MAE & RMSE & SMAPE & MAE & RMSE & SMAPE & MAE & RMSE & SMAPE \\
        \midrule
        \framework & \textbf{42.74} & \textbf{63.11} & \textbf{0.78} & \textbf{46.80} & \textbf{68.34} & \textbf{0.81} & \textbf{21.95} & \textbf{37.20} & \textbf{0.50} & \textbf{24.08} & \textbf{39.34} & \textbf{0.54} \\
        \midrule
        \multicolumn{13}{l}{\textit{Physical Mechanisms}} \\
        \midrule
        w/o Diffusion (\framework\texttt{-D}) & 44.31 & 65.25 & 0.80 & 48.37 & 70.48 & 0.83 & 23.52 & 39.34 & 0.52 & 25.65 & 41.48 & 0.56 \\
        w/o Advection (\framework\texttt{-A}) & 42.81 & 66.34 & 0.78 & 48.80 & 69.64 & 0.84 & 23.95 & 40.02 & 0.52 & 24.98 & 40.74 & 0.58 \\
        w/o Reaction (\framework\texttt{-R}) & 45.15 & 66.80 & 0.82 & 49.60 & 71.90 & 0.85 & 24.10 & 40.50 & 0.53 & 26.80 & 42.90 & 0.59 \\
        w/o Physics (\framework\texttt{-P}) & 47.85 & 69.50 & 0.86 & 53.20 & 75.40 & 0.90 & 26.50 & 43.80 & 0.58 & 29.10 & 46.50 & 0.63 \\
        \midrule
        \multicolumn{13}{l}{\textit{Models \& Fusion}} \\
        \midrule
        w/o Kalman (\framework\texttt{-K}) & 56.40 & 78.20 & 0.95 & 63.50 & 85.10 & 1.05 & 30.20 & 48.50 & 0.65 & 35.50 & 54.20 & 0.72 \\
        w/ EnKF (\framework\texttt{-EnKF})  
        & 43.62 & 62.45 & 0.81 
        & 47.85 & 67.20 & 0.84     
        & 24.08 & 35.12 & 0.59 
        & 25.07 & 36.06 & 0.61 \\
        w/o Neur. Corr (\framework\texttt{-N}) & 43.10 & 63.37 & 0.79 & 47.16 & 68.60 & 0.82 & 22.31 & 37.46 & 0.51 & 24.44 & 39.60 & 0.55 \\
        w/o Mamba (\framework\texttt{-M}) & 43.05 & 63.80 & 0.79 & 51.50 & 74.20 & 0.88 & 22.15 & 37.60 & 0.51 & 28.40 & 45.20 & 0.61 \\
        w/ PatchTST (\framework\texttt{-Patch}) 
        & 44.18 & 64.50 & 0.81 
        & 48.42 & 69.85 & 0.84 
        & 24.60 & 36.15 & 0.61 
        & 25.55 & 37.10 & 0.63 \\
        w/ DLinear (\framework\texttt{-DLinear}) 
        & 44.35 & 64.20 & 0.81 
        & 48.55 & 69.65 & 0.85 
        & 24.8 & 35.88 & 0.61 
        & 25.64 & 36.94 & 0.63 \\
        w/ iTrans (\framework\texttt{-iTrans}) 
        & 44.05 & 64.40 & 0.80 
        & 48.30 & 69.70 & 0.84 
        & 24.72 & 36.05 & 0.60 
        & 25.48 & 37.05 & 0.62 \\
        w/ Crossformer (\framework\texttt{-Cross}) 
        & 43.95 & 64.30 & 0.80 
        & 48.20 & 69.55 & 0.84 
        & 24.55 & 36.20 & 0.60 
        & 25.42 & 37.20 & 0.62 \\
        Late Fusion & 53.25 & 76.10 & 0.92 & 59.80 & 81.50 & 0.98 & 29.40 & 47.10 & 0.63 & 33.80 & 51.50 & 0.69 \\
        \bottomrule
    \end{tabular}
    \vspace{-17pt}
\end{table*}

Tables~\ref{tab:know_results} and~\ref{tab:hanoi_results} report the main PM\textsubscript{2.5} forecasting results on KnowAir and HanoiAir. Overall, \framework is most effective in long-horizon and sudden-change scenarios, where open-loop autoregressive and purely data-driven models are more prone to drift or over-smoothing. Results for O\textsubscript{3} are provided in Appx.~\ref{app:subsec:more_exp}.

Transformer-based and SSM-based models, such as \texttt{TimeMoE}, \texttt{PatchTST}, and \texttt{S-Mamba}, achieve competitive short-term accuracy, but their errors generally increase as the horizon becomes longer. In contrast, \framework maintains more stable long-horizon performance. This supports our central design: the UDE-based process branch provides a structured physical prior, while Neural Kalman Fusion repeatedly corrects the latent state during rollout, reducing the error accumulation typical of open-loop forecasting.

The advantage of \framework is more evident on \textbf{HanoiAir}, which has fewer monitoring stations and noisier regional dynamics. Compared with strong direct-mapping models such as \texttt{TimeMoE} and \texttt{ModernTCN}, \framework achieves more stable performance, suggesting that closed-loop neural state correction improves robustness when observations are sparse or noisy.

Sudden pollution changes are particularly challenging because many baselines tend to smooth rapid concentration shifts. In the Sudden Change setting, \framework achieves stronger performance than Transformer and SSM baselines. This improvement is consistent with the role of the meteorology-gated neural reaction module, which conditions chemical dynamics on environmental drivers such as radiation and temperature, helping the model capture nonlinear secondary pollution formation. Probabilistic forecasting results on \textbf{KnowAir} PM\textsubscript{2.5} are shown in Appx.~\ref{app:subsec:probabilistic}.


\subsection{Ablation Study}
\label{sec:ablation}


We conduct ablation studies on \textbf{KnowAir} and \textbf{HanoiAir} to evaluate two aspects of \framework: (i) the physical mechanisms in the process branch and (ii) the architecture of the virtual measurement branch and fusion module. Table~\ref{tab:combined_ablation} shows the experimental results of the ablation studies.

\vspace{-9pt}
\paragraph{Effect of Physical Mechanisms.}
\label{sec:ablation_physics}


We first remove individual components from the UDE drift function in Eq.~\ref{eq:ude_master}: isotropic diffusion (\framework{\texttt{-D}}), anisotropic advection (\framework{\texttt{-A}}), the meteorology-gated neural reaction module (\framework{\texttt{-R}}), and all explicit physical terms (\framework{\texttt{-P}}). Removing all physical terms leads to clear long-horizon degradation, confirming that the structured process branch helps constrain rollout dynamics. Among the physical components, removing the reaction module causes substantial degradation, showing that meteorology-conditioned chemical dynamics are important for modeling secondary pollutant formation beyond transport alone.

\vspace{-9pt}
\paragraph{Effect of Virtual Measurements and Fusion.}
\label{sec:ablation_arch}

We next evaluate the virtual measurement branch and the closed-loop fusion design. We replace \texttt{S-Mamba} with a bidirectional \texttt{LSTM} (\framework{\texttt{-M}}) and with representative temporal encoders, including \texttt{PatchTST} (\framework{\texttt{-Patch}}), \texttt{DLinear} (\framework{\texttt{-DLinear}}), \texttt{iTransformer} (\framework{\texttt{-iTrans}}), and \texttt{Crossformer} (\framework{\texttt{-Cross}}). We also remove posterior feedback to obtain an open-loop UDE solver (\framework{\texttt{-K}}), remove the learnable correction term while keeping the diagonal Kalman update (\framework{\texttt{-N}}), replace Neural Kalman Fusion with an Ensemble Kalman Filter (\framework{\texttt{-EnKF}}), and decouple the process and virtual measurement branches using weighted output-level averaging (\texttt{Late Fusion}).
We observe that removing closed-loop feedback causes the largest degradation, especially at longer horizons, confirming that correcting the latent state before the next rollout step is essential for stability. Removing the learnable neural correction also degrades performance, suggesting that a strictly Gaussian correction is insufficient for open environmental systems with boundary leakage, unobserved emissions, and model misspecification. Finally, \texttt{Late Fusion} performs worse than latent-state fusion, supporting our claim that the two branches should correct each other during rollout rather than being combined only at the output level.
\setlength{\parskip}{0pt}
\subsection{Effect of Forecast Horizons}
We show that \framework maintains significantly more stable long-term behavior than physics-guided baselines. Fig.~\ref{fig:horizon_analysis} illustrates the error evolution across varying horizons. \framework exhibits superior stability compared to existing physics-guided baselines (e.g., \texttt{AirPhyNet}, \texttt{Air-DualODE}), with the performance gap widening significantly as the prediction window extends. While prior methods suffer from rapid error propagation, \framework maintains a flattened error trajectory, validating the effectiveness of our closed-loop design in mitigating long-term physical drift.
\begin{figure*}[t]
    \centering
    \footnotesize

    \begin{minipage}{\textwidth}
        \centering

        \begin{subfigure}[b]{0.32\textwidth}
            \begin{tikzpicture}
                \begin{axis}[
                    width=\linewidth, height=4.5cm,
                    xlabel={Forecast Horizon (Hours)},
                    ylabel={MAE},
                    grid=major, grid style={dashed, gray!30},
                    xmin=6, xmax=216,
                    xtick={6,72,144,216},
                    ylabel style={yshift=-5pt},
                    legend style={draw=none} 
                ]
                    \addplot[color=mygray, thick, mark=x] table [x=Horizon, y=GTS] {data/mae_horizon.txt};
                    \addplot[color=mypurple, thick, mark=diamond] table [x=Horizon, y=AirPhyNet] {data/mae_horizon.txt};
                    \addplot[color=mygreen, thick, mark=triangle] table [x=Horizon, y=PatchTST] {data/mae_horizon.txt};
                    \addplot[color=myorange, thick, mark=square] table [x=Horizon, y=DLinear] {data/mae_horizon.txt};
                    \addplot[color=myblue, thick, mark=*] table [x=Horizon, y=TimeMoE] {data/mae_horizon.txt};
                    \addplot[color=myred, thick, mark=pentagon*] table [x=Horizon, y=AeroChem] {data/mae_horizon.txt};
                \end{axis}
            \end{tikzpicture}
            \caption{MAE}
        \end{subfigure}
        \hfill
        \begin{subfigure}[b]{0.32\textwidth}
            \begin{tikzpicture}
                \begin{axis}[
                    width=\linewidth, height=4.5cm,
                    xlabel={Forecast Horizon (Hours)},
                    ylabel={RMSE},
                    grid=major, grid style={dashed, gray!30},
                    xmin=6, xmax=216,
                    xtick={6,72,144,216},
                    legend style={draw=none}
                ]
                    \addplot[color=mygray, thick, mark=x] table [x=Horizon, y=GTS] {data/rmse_horizon.txt};
                    \addplot[color=mypurple, thick, mark=diamond] table [x=Horizon, y=AirPhyNet] {data/rmse_horizon.txt};
                    \addplot[color=mygreen, thick, mark=triangle] table [x=Horizon, y=PatchTST] {data/rmse_horizon.txt};
                    \addplot[color=myorange, thick, mark=square] table [x=Horizon, y=DLinear] {data/rmse_horizon.txt};
                    \addplot[color=myblue, thick, mark=*] table [x=Horizon, y=TimeMoE] {data/rmse_horizon.txt};
                    \addplot[color=myred, thick, mark=pentagon*] table [x=Horizon, y=AeroChem] {data/rmse_horizon.txt};
                \end{axis}
            \end{tikzpicture}
            \caption{RMSE}
        \end{subfigure}
        \hfill
        \begin{subfigure}[b]{0.32\textwidth}
            \begin{tikzpicture}
                \begin{axis}[
                    width=\linewidth, height=4.5cm,
                    xlabel={Forecast Horizon (Hours)},
                    ylabel={SMAPE},
                    grid=major, grid style={dashed, gray!30},
                    xmin=6, xmax=216,
                    xtick={6,72,144,216},
                    legend style={draw=none}
                ]
                    \addplot[color=mygray, thick, mark=x] table [x=Horizon, y=GTS] {data/smape_horizon.txt};
                    \addplot[color=mypurple, thick, mark=diamond] table [x=Horizon, y=AirPhyNet] {data/smape_horizon.txt};
                    \addplot[color=mygreen, thick, mark=triangle] table [x=Horizon, y=PatchTST] {data/smape_horizon.txt};
                    \addplot[color=myorange, thick, mark=square] table [x=Horizon, y=DLinear] {data/smape_horizon.txt};
                    \addplot[color=myblue, thick, mark=*] table [x=Horizon, y=TimeMoE] {data/smape_horizon.txt};
                    \addplot[color=myred, thick, mark=pentagon*] table [x=Horizon, y=AeroChem] {data/smape_horizon.txt};
                \end{axis}
            \end{tikzpicture}
            \caption{SMAPE}
        \end{subfigure}

        \vspace{0.2cm}
        \begin{tikzpicture}
            \begin{axis}[
                hide axis,
                xmin=0, xmax=1,
                ymin=0, ymax=1,
                legend columns=3,
                legend style={
                    font=\footnotesize,
                    draw=none,
                    column sep=12pt,
                    at={(0.5,0.5)},
                    anchor=center
                }
            ]
                \addlegendimage{color=mygray, mark=x, thick}
                \addlegendentry{GTS}
                \addlegendimage{color=mypurple, mark=diamond, thick}
                \addlegendentry{AirPhyNet}
                \addlegendimage{color=mygreen, mark=triangle, thick}
                \addlegendentry{PatchTST}
                \addlegendimage{color=myorange, mark=square, thick}
                \addlegendentry{DLinear}
                \addlegendimage{color=myblue, mark=*, thick}
                \addlegendentry{TimeMoE}
                \addlegendimage{color=myred, mark=pentagon*, thick}
                \addlegendentry{AeroChem}
            \end{axis}
        \end{tikzpicture}

    \end{minipage}

    \caption{Performance across forecast horizons.}
    \label{fig:horizon_analysis}
\end{figure*}
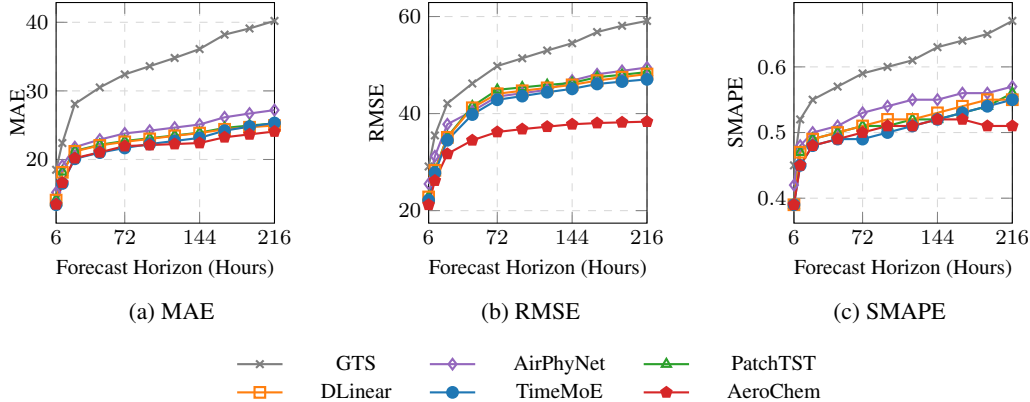
\setlength{\parskip}{0pt}
\subsection{Hyperparameter Studies}
\label{app:sec:hyperparameter}
We assess the model's robustness to three key hyperparameters. Fig.~\ref{fig:sensitivity} shows the experimental results.
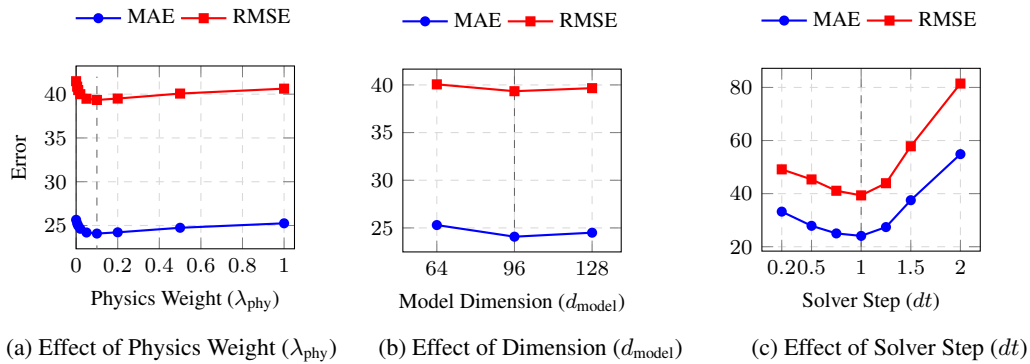
\begin{figure*}[t]
    \centering
    \pgfplotsset{
        every axis/.style={
            width=\linewidth,
            height=4cm,
            grid=major,
            grid style={dashed,gray!30},
            xlabel near ticks,
            ylabel near ticks,
            label style={font=\footnotesize},
            tick label style={font=\footnotesize},
            legend style={
                at={(0.5,1.15)}, 
                anchor=south, 
                legend columns=-1,
                draw=none, 
                font=\footnotesize,
                fill=none
            },
            mark size=1.5pt
        }
    }

    \begin{subfigure}[b]{0.32\textwidth}
        \begin{tikzpicture}
            \begin{axis}[
                xlabel={Physics Weight ($\lambda_{\text{phy}}$)},
                ylabel={Error},
                xmin=0, xmax=1.05,
                xtick={0, 0.2, 0.4, 0.6, 0.8, 1.0},
            ]
            
            \addplot[color=blue, mark=*, thick] coordinates {
                (0.0, 25.65) (0.005, 25.24) (0.01, 24.97) (0.02, 24.63) 
                (0.05, 24.20) (0.1, 24.08) (0.2, 24.22) (0.5, 24.74) (1.0, 25.25)
            };
            \addlegendentry{MAE}

            \addplot[color=red, mark=square*, thick] coordinates {
                (0.0, 41.48) (0.005, 40.86) (0.01, 40.47) (0.02, 40.00) 
                (0.05, 39.48) (0.1, 39.34) (0.2, 39.50) (0.5, 40.07) (1.0, 40.63)
            };
            \addlegendentry{RMSE}
            
            \draw[dashed, black!60] (axis cs:0.1, 20) -- (axis cs:0.1, 42);
            \end{axis}
        \end{tikzpicture}
        \caption{Effect of Physics Weight ($\lambda_{\text{phy}}$)}
        \label{fig:sens_phy}
    \end{subfigure}
    \hfill
    \begin{subfigure}[b]{0.32\textwidth}
        \begin{tikzpicture}
            \begin{axis}[
                xlabel={Model Dimension ($d_{\text{model}}$)},
                xtick={64, 96, 128},
                xmin=50, xmax=140,
            ]
            \addplot[color=blue, mark=*, thick] coordinates {
                (64, 25.30) (96, 24.08) (128, 24.50)
            };
            \addlegendentry{MAE}

            \addplot[color=red, mark=square*, thick] coordinates {
                (64, 40.06) (96, 39.34) (128, 39.66)
            };
            \addlegendentry{RMSE}
            
            \draw[dashed, black!60] (axis cs:96, 20) -- (axis cs:96, 42);
            \end{axis}
        \end{tikzpicture}
        \caption{Effect of Dimension ($d_{\text{model}}$)}
        \label{fig:sens_dim}
    \end{subfigure}
    \hfill
    \begin{subfigure}[b]{0.32\textwidth}
        \begin{tikzpicture}
            \begin{axis}[
                xlabel={Solver Step ($dt$)},
                xmin=0, xmax=2.2,
                xtick={0.2, 0.5, 1.0, 1.5, 2.0},
            ]
            \addplot[color=blue, mark=*, thick] coordinates {
                (0.2, 33.22) (0.5, 27.88) (0.75, 25.01) 
                (1.0, 24.08) (1.25, 27.40) (1.5, 37.52) (2.0, 54.87)
            };
            \addlegendentry{MAE}

            \addplot[color=red, mark=square*, thick] coordinates {
                (0.2, 49.14) (0.5, 45.35) (0.75, 41.08) 
                (1.0, 39.34) (1.25, 43.92) (1.5, 57.84) (2.0, 81.46)
            };
            \addlegendentry{RMSE}
            
            \draw[dashed, black!60] (axis cs:1.0, 20) -- (axis cs:1.0, 85);
            \end{axis}
        \end{tikzpicture}
        \caption{Effect of Solver Step ($dt$)}
        \label{fig:sens_dt}
    \end{subfigure}

    \caption{Sensitivity analysis on \textbf{KnowAir} dataset. The dashed vertical lines indicate the default hyperparameters used in the full model ($\lambda=0.1, d=96, dt=1.0$), which consistently correspond to the optimal performance regions (lowest error).}
    \label{fig:sensitivity}
\end{figure*}
\setlength{\parskip}{0pt}
\textbf{Physics Weight ($\lambda_{\text{phy}}$):} Fig.~\ref{fig:sens_phy} reveals a distinct trade-off between physical constraints and data flexibility, forming a bell-shaped performance curve. When $\lambda_{\text{phy}} \to 0$, weak physical guidance increases long-term drift and prediction errors. Conversely, $\lambda_{\text{phy}} > 1$ over-constrains latent dynamics, limiting the model's ability to capture localized variations. The best performance at $\lambda_{\text{phy}} = 0.1$ balances physical consistency with data-driven adaptability.
\textbf{Model Dimension ($d_{\text{model}}$):} As depicted in Fig~\ref{fig:sens_dim}, expanding the hidden dimension initially yields substantial gains, confirming that sufficient representational capacity is requisite for encoding complex spatiotemporal dependencies. However, performance saturates beyond $d_{\text{model}}=96$, where further increases offer negligible improvements while inflating computational costs. Consequently, we select 96 as the optimal equilibrium between model expressiveness and computational efficiency.
\textbf{Solver Step ($dt$):} Fig.~\ref{fig:sens_dt} demonstrates the model's remarkable numerical robustness, maintaining consistent accuracy across varying step sizes. Typically, coarser discretization introduces integration errors in standard ODE solvers. However, our results indicate that the fusion mechanism effectively rectifies these trajectory deviations by assimilating data-driven corrections.
\setlength{\parskip}{0pt}
\subsection{Additional Experiments}
First, given that \framework utilizes an ODE solver to model continuous dynamics and incorporates complex spatio-temporal dependencies, it is necessary to verify the model's numerical stability. We study the numerical stability in Appx.~\ref{app:sec:numerical stability}. 
Second, additional visualizations of model predictions during sudden pollution events are shown in Appx.~\ref{app:sec:sudden_change}, highlighting \framework’s ability to track abrupt shifts while revealing remaining challenges in recovering extreme peak amplitudes. Further, to validate the interpretability and robustness of \framework, we go beyond standard metrics to analyze the internal mechanisms of the model. We specifically examine two critical aspects: the learning of chemical kinetics and the stability of long-term forecasting. The case study is presented in Appx.~\ref{app:subsec:case_study}.



\section{Conclusion}
\label{sec:conclusion}

We presented \framework, a closed-loop physics-informed state space framework for long-term chemically reactive air quality forecasting. \framework couples a UDE-based process branch with a Mamba-based virtual measurement branch through Neural Kalman Fusion, enabling the two branches to correct each other during rollout. By updating the latent physical state before each subsequent prediction step, \framework mitigates the long-horizon drift caused by open-loop forecasting designs.
Experiments on \textbf{KnowAir} and the newly constructed \textbf{HanoiAir} dataset demonstrate that \framework improves long-horizon stability and sudden-change forecasting, especially under noisy observations and secondary pollutant formation. These findings highlight closed-loop neural state correction as a promising direction for robust physics-informed environmental forecasting. Future work will extend the framework to richer multi-species chemical interaction modeling and probabilistic risk assessment for environmental decision-making.

\clearpage
\clearpage
\section*{Acknowledgments}
This work was supported by the Center for Environmental Intelligence, VinUniversity, under project VUNI.CEI.FS\_0011. The authors are grateful for the constructive feedback and insightful discussions provided by the anonymous reviewers of NeurIPS 2026.
\section*{Broader Impact.}
Accurate long-term air quality forecasting has direct societal relevance, as it can support public-health planning, emission-control strategies, and early-warning systems for hazardous pollution events. By reducing long-horizon drift through closed-loop state correction, \framework may help environmental agencies better anticipate persistent haze episodes, secondary pollutant formation, and sudden pollution changes, especially in rapidly urbanizing or data-sparse regions. More broadly, this work contributes to scientific machine learning by showing how physics-informed dynamics and data-driven measurements can be coupled in a closed-loop co-correction framework for partially observed open systems. Similar ideas may be useful beyond air quality forecasting, including climate modeling, hydrology, energy systems, and other domains where imperfect physical simulators must be corrected using observational data.

\section*{Limitations.}
Several limitations remain. First, although the physics-informed structure improves stability, the underlying process model is still a simplified approximation of atmospheric chemistry and transport. The learned reaction operator captures effective dynamics rather than explicitly resolving full chemical mechanisms, which may limit extrapolation to unseen regimes. Second, uncertainty modeling is based on learned diagonal covariances, which do not capture full spatial correlations and may underestimate structured error. Finally, the framework relies on the availability of meteorological and emission inputs; inaccuracies or delays in these auxiliary data can propagate through the system. 

\bibliographystyle{plain}
\small
\bibliography{ref}


\clearpage
\counterwithin{figure}{section}
\counterwithin{table}{section}

\clearpage
\appendix
\section{Appendix}
\subsection{Algorithmic Summary of \framework}
\label{app:subsec:algorithm}

We provide a step-by-step algorithmic summary of the \framework framework.
Algorithm~\ref{alg:aerochem} details the closed-loop forecasting procedure, explicitly illustrating how the measurement model, the physics-informed Universal Differential Equation (UDE), and the Neural Kalman Fusion module interact over the forecasting horizon.
The algorithm highlights the recursive feedback mechanism in which posterior latent states are assimilated and fed back into the UDE solver.

\begin{algorithm}[h]
\small
\caption{\framework Forecasting}
\label{alg:aerochem}
\DontPrintSemicolon
\SetKwInOut{Input}{Input}
\SetKwInOut{Output}{Output}

\Input{Sensor graph $\mathcal{G}$; historical observations $\mathbf{X}_{1:T_{\text{in}}}$; meteorology $\mathbf{M}_{1:T_{\text{out}}}$; emissions $\mathbf{E}_{1:T_{\text{out}}}$; forecast horizon $T_{\text{out}}$.}
\Output{Forecasted observations $\hat{\mathbf{X}}_{T_{\text{in}}+1:T_{\text{in}}+T_{\text{out}}}$.}

\BlankLine
\textbf{Measurement encoding:}\;
$\mathbf{H}_{\text{ctx}} \leftarrow \text{S-Mamba}(\mathbf{X}_{1:T_{\text{in}}})$\;
$(\mathbf{Z}_{\text{base}}, \mathbf{R}_{\text{base}}) \leftarrow \text{ProjHeads}(\mathbf{H}_{\text{ctx}})$\;

\BlankLine
\textbf{Initialization:}\;
Initialize posterior latent state $\mathbf{S}^+_{0}$ (e.g., from the last observed timestep via a projection/encoder)\;

\BlankLine
\For{$t \leftarrow 1$ \KwTo $T_{\text{out}}$}{
    \textbf{Adaptive virtual observation:}\\
    $\mathbf{Z}_{\text{obs},t} \leftarrow \mathbf{Z}_{\text{base}} + \alpha \tanh(\mathbf{S}^+_{t-1})$\\
    $\mathbf{R}_{\text{obs},t} \leftarrow \text{Softplus}(\mathbf{R}_{\text{base}})\left(1 + \beta \frac{t}{T_{\text{out}}}\right)$

    \BlankLine
    \textbf{UDE prediction (time update):}\\
    $\mathbf{S}^-_t \leftarrow \text{UDE-Solve}\!\left(\mathbf{S}^+_{t-1}; \mathbf{M}_t, \mathbf{E}_t, \mathcal{G}\right)$\\
    $\mathbf{P}^-_t \leftarrow \text{Softplus}(\text{Linear}(\mathbf{S}^-_t))$

    \BlankLine
    \textbf{Neural Kalman fusion (measurement update):}\\
    $\mathbf{K}_t \leftarrow \mathbf{P}^-_t \oslash (\mathbf{P}^-_t + \mathbf{R}_{\text{obs},t})$\\
    $\boldsymbol{\epsilon}_t \leftarrow \mathbf{Z}_{\text{obs},t} - \mathbf{S}^-_t$\\
    $\boldsymbol{\epsilon}'_t \leftarrow \boldsymbol{\epsilon}_t + \text{Linear}\!\left([\mathbf{S}^-_t, \mathbf{Z}_{\text{obs},t}]\right)$\\
    $\mathbf{S}^+_t \leftarrow \mathbf{S}^-_t + \mathbf{K}_t \odot \boldsymbol{\epsilon}'_t$\\
    $\mathbf{P}^+_t = (\mathbf{1} - \mathbf{K}_t) \odot \mathbf{P}^-_t$
}


\BlankLine
\textbf{Decoding (Probabilistic Output):}\;
\For{$t \leftarrow 1$ \KwTo $T_{\text{out}}$}{
    $\hat{\mathbf{X}}_{T_{\text{in}}+t} \leftarrow \mathcal{D}_{\text{out}}(\mathbf{S}^+_t)$ 
    $\boldsymbol{\sigma}^2_{T_{\text{in}}+t} \leftarrow \text{Softplus}(\mathcal{D}_{\text{var}}(\mathbf{P}^+_t))$ 
}

\Return $\{\hat{\mathbf{X}}_{T_{\text{in}}+t}, \boldsymbol{\sigma}^2_{T_{\text{in}}+t}\}_{t=1}^{T_{\text{out}}}$\;
\end{algorithm}

\subsection{Variational Interpretation of Neural Kalman Fusion}
\label{app:variational}

We first recast \framework under a single variational objective. Let $\Sm_{1:T_{\mathrm{out}}} \in \R^{T_{\mathrm{out}}\times N\times D}$ denote the latent trajectory and let $\mathbf{X}_{1:T_{\mathrm{in}}}$ denote the historical observations. Define the per-step variational free energy
\begin{equation}
\mathcal{F}_{t}(\Sm)
\;=\;
\tfrac{1}{2}\,\big(\Sm-\Sprior\big)^{\!\top}(\Pprior)^{-1}\big(\Sm-\Sprior\big)
\;+\;
\tfrac{1}{2}\,\big(\Zobs-\Sm\big)^{\!\top}\Robs^{-1}\big(\Zobs-\Sm\big)
\;+\;
\Psi_{t}(\Sm),
\label{eq:free-energy}
\end{equation}
where $(\Sprior,\Pprior)$ is produced by the UDE prediction step (Eq.~\ref{eq:softplus}), $(\Zobs,\Robs)$ is the virtual measurement of the Mamba encoder (Eq.~\ref{eq:zobs_robs}), and $\Psi_{t}:\R^{N\times D}\to\R$ is a smooth residual potential that absorbs all non-Gaussian discrepancies of the open system (boundary leakage, unobserved emissions, model misspecification). We assume $\Pprior,\Robs$ are positive element-wise (guaranteed by the \texttt{Softplus} activations Eq \ref{eq:softplus} and treat all matrix--vector products in Eq.~\ref{eq:free-energy} element-wise to be consistent with the diagonal covariance approximation $\Pprior,\Robs\in\R_{>0}^{N\times D}$ used throughout \framework.

\begin{lemma}
\label{lem:newton}
Assume $\Psi_{t}\in C^{3}(\R^{N\times D})$. Let $\Sm^{\star}_{t}$ be a local minimiser of $\mathcal{F}_{t}$ and let $\Spost$ be one Newton step initialised at the prior mean $\Sprior$:
\begin{equation}
\Spost \;=\; \Sprior - \big[\nabla^{2}\mathcal{F}_{t}(\Sprior)\big]^{-1}
\nabla\mathcal{F}_{t}(\Sprior).
\label{eq:newton}
\end{equation}
Then, expanding under the diagonal-covariance regime:
\begin{enumerate}
\item[\emph{(i)}] If $\Psi_{t}\equiv 0$ (purely Gaussian), then Eq.~\ref{eq:newton}
reduces \emph{exactly} to the standard Kalman update of \framework, as shown in the fifth equation of \ref{eq:correction})
\item[\emph{(ii)}] If $\Psi_{t}\not\equiv 0$, the Newton step admits a
first-order Taylor expansion of $\nabla\Psi_{t}$ around the joint point
$(\Sprior,\Zobs)$. The resulting correction is bilinear in
$(\Sprior,\Zobs)$ and is therefore representable by an affine map
$\mathrm{Linear}\big([\Sprior,\Zobs]\big)$. Substituting yields exactly
the Neural Kalman Fusion update (Eq \ref{eq:correction}).
The residual error of the Newton step is
$\Spost - \Sm^{\star}_{t} = O\!\big(\norm{\Sprior-\Sm^{\star}_{t}}^{2}\big)$.
\end{enumerate}
\end{lemma}
 
\begin{proof}[Proof:]
\textbf{Part (i).} With $\Psi_{t}\equiv 0$,
$\mathcal{F}_{t}$ is quadratic. Differentiating Eq.~\ref{eq:free-energy}:
\[
\nabla\mathcal{F}_{t}(\Sprior)
\;=\;(\Pprior)^{-1}(\Sprior-\Sprior)\,-\,\Robs^{-1}(\Zobs-\Sprior)
\;=\;-\Robs^{-1}\,\eps_{t},
\]
\[
\nabla^{2}\mathcal{F}_{t}(\Sprior) \;=\;(\Pprior)^{-1}+\Robs^{-1}.
\]
Under the diagonal regime
$(\Pprior)^{-1}+\Robs^{-1}=(\Pprior+\Robs)\oslash(\Pprior\odot\Robs)$,
so the Newton step in Eq.~\ref{eq:newton} becomes
\[
\Spost
\;=\;\Sprior + \big[(\Pprior)^{-1}+\Robs^{-1}\big]^{-1}\Robs^{-1}\,\eps_{t}
\;=\;\Sprior + \underbrace{\big(\Pprior\oslash(\Pprior+\Robs)\big)}_{=\,\Kt}\odot\,\eps_{t},
\]
which is the fourth equation of Eq.~\ref{eq:correction}. The covariance update $\Pprior\to\Ppost$ of fifth equation in  Eq.~\ref{eq:correction}
follows from the standard Kalman posterior covariance identity 
under optimal gain \cite{DBLP:journals/tsp/RevachSNESE22}.
 
\textbf{Part (ii).} Adding $\Psi_{t}$ contributes
$\nabla\Psi_{t}(\Sprior)$ to the gradient and
$\nabla^{2}\Psi_{t}(\Sprior)$ to the Hessian. Taylor-expanding
$\nabla\Psi_{t}$ to first order around the joint argument
$(\Sprior,\Zobs)$:
\[
\nabla\Psi_{t}(\Sprior)
\;\approx\;
\mathbf{a}+\mathbf{B}\Sprior+\mathbf{C}\Zobs
\;=\;\mathrm{Linear}\big([\Sprior,\Zobs]\big),
\]
for parameter tensors $(\mathbf{a},\mathbf{B},\mathbf{C})$ that are
absorbed into the linear layer of the third equation in Eq.~\ref{eq:correction}. The
second-order Hessian contribution can be merged into $\Pprior$ via
covariance reparameterisation. The Newton-step quadratic convergence
$\Spost-\Sm^{\star}_{t}=O(\norm{\Sprior-\Sm^{\star}_{t}}^{2})$
is the standard result for $C^{3}$ objectives with positive-definite
Hessian~\cite{DBLP:conf/nips/ZimmermannNM24}, \cite{DBLP:books/sp/NocedalW99}.
\end{proof}

\subsection{Closed-loop Neural State Correction}
\label{app:errorinlongterm}
We show that the closed-loop neural state correction yields a posterior error that is uniformly bounded in $t$, while the open-loop variant (\emph{\framework-K} in Table~\ref{tab:combined_ablation}) admits exponential divergence, details are included in App \ref{app:errorinlongterm}. Let $\Phistep:\R^{N\times D}\!\to\!\R^{N\times D}$ denote the one-step UDE flow so that $\Sprior\!=\!\Phistep(\Spostprev)$. Define the posterior error $\err\!\coloneqq\!\Spost-\Strue$ relative to the unobserved true latent trajectory $\Strue$.
 
\begin{assumption}\label{ass:stability}
\textnormal{(A5)} $\Fdrift$ is $L_{F}$-Lipschitz in $\Sm$, so $\norm{\Phistep(\Sm_{1})-\Phistep(\Sm_{2})}_{2}\!\le\!(1+L_{F}\Delta t)\norm{\Sm_{1}-\Sm_{2}}_{2}$;\; \textnormal{(A6)} \emph{(by construction)} the gain in Eq.~\ref{eq:correction} satisfies $\Kt\!\in\![\kappa_{\min},\kappa_{\max}]$ element-wise with $0\!<\!\kappa_{\min}\!\le\!\kappa_{\max}\!<\!1$;\;
\textnormal{(A7)} \emph{(by construction)} the correction in the third equation of Eq.~\ref{eq:correction} is $\LinOp([\Sprior,\Zobs])\!=\!W\eps_{t}$ with $\norm{W}_{2}\!\le\!L_{c}$ enforced by spectral normalisation;\; \textnormal{(A8)} the conditional measurement bias is bounded, $\norm{\E[\Zobs\!\mid\!\Filt_{t-1}]-\Strue}_{2}\!\le\!b$;\; \textnormal{(A9)} the centred process and measurement noises satisfy $\E\norm{\wt}_{2}^{2}\!\le\!\sigma_{W}^{2}\Delta t$ and $\E\norm{\nut-\E[\nut\!\mid\!\Filt_{t-1}]}_{2}^{2}\!\le\!\sigma_{\nu}^{2}$, with $\wt$ and the centred residual conditionally independent given $\Filt_{t-1}$.
\end{assumption}
 
\begin{theorem}[Closed-loop versus open-loop posterior error]
\label{thm:stability}
Define the contraction factor
\(
    \alpha\,\coloneqq\,(1+L_{F}\Delta t)[(1-\kappa_{\min})+\kappa_{\max}L_{c}].
\)
If $\alpha\!<\!1$, then under Assumption~\ref{ass:stability} the closed-loop posterior error of \framework satisfies, for every $t\!\in\!\{1,\dots,T_{\mathrm{out}}\}$,
\begin{equation}
\sqrt{\E\norm{\err}_{2}^{2}}
\;\le\;
\alpha^{t}\sqrt{\E\norm{\mathbf{e}_{0}}_{2}^{2}}+\frac{V}{1-\alpha},
\qquad
V=\kappa_{\max}b+\sqrt{\kappa_{\max}^{2}\sigma_{\nu}^{2}+(1\!-\!\kappa_{\min}\!+\!\kappa_{\max}L_{c})^{2}\sigma_{W}^{2}\Delta t},
\label{eq:closed-bound}
\end{equation}
which is uniformly bounded in $t$. The open-loop variant ($\Kt\!\equiv\!\mathbf{0}$, \emph{\framework-K}) admits the matching lower bound
\begin{equation}
\E\norm{\mathbf{e}_{t}^{\mathrm{open}}}_{2}^{2}
\;\ge\;(1+L_{F}\Delta t)^{2t}\E\norm{\mathbf{e}_{0}}_{2}^{2},
\label{eq:open-bound}
\end{equation}
which grows exponentially in $t$ whenever $L_{F}\!>\!0$.
\end{theorem}

We give the complete proof in six steps. Throughout, $\nut\!\coloneqq\!\Zobs-\Strue$ denotes the measurement residual, $\wt\!\coloneqq\!\Strue-\Phistep(\Strueprev)$ the discretised Wiener increment, $\nutbar\!\coloneqq\!\E[\nut\!\mid\!\Filt_{t-1}]$ the conditional bias, and $\nuttilde\!\coloneqq\!\nut-\nutbar$ the centred residual. By Assumptions~A8--A9, $\norm{\nutbar}_{2}\!\le\!b$, $\E\norm{\nuttilde}_{2}^{2}\!\le\!\sigma_{\nu}^{2}$, and $\E\norm{\wt}_{2}^{2}\!\le\!\sigma_{W}^{2}\Delta t$, with $\wt,\nuttilde$
conditionally zero-mean and independent of $\errprev$.
 
\paragraph{Step~1}
Substitute $\Spost$ into $\err\!=\!\Spost-\Strue$ and use
$\Strue\!=\!\Phistep(\Strueprev)+\wt$, $\Zobs\!=\!\Strue+\nut$,
$\Sprior\!=\!\Phistep(\Spostprev)$:
\begin{equation}
\err
=(\mathbf{1}-\Kt)\odot\bigl[\Phistep(\Spostprev)-\Phistep(\Strueprev)\bigr]
-(\mathbf{1}-\Kt)\odot\wt
+\Kt\odot\nut
+\Kt\odot W\eps_{t}.
\label{eq:err-raw}
\end{equation}
By the second and third equations in ~\ref{eq:correction}, the innovation expands as
\begin{equation}
\eps_{t}=\Zobs-\Sprior=\nut+(\Strue-\Sprior)
=\nut-\bigl[\Phistep(\Spostprev)-\Phistep(\Strueprev)\bigr]+\wt.
\label{eq:eps-expand}
\end{equation}
Substituting Eq.~\ref{eq:eps-expand} into the last term of Eq.~\ref{eq:err-raw} and collecting the coefficients of
$\Phistep(\Spostprev)-\Phistep(\Strueprev)$, $\wt$ and $\nut$:
\begin{equation}
\err
\;=\;
\underbrace{\bigl[(\mathbf{1}-\Kt)-\Kt W\bigr]\odot\bigl[\Phistep(\Spostprev)-\Phistep(\Strueprev)\bigr]}_{=:\,T_{1}}
+\underbrace{\bigl[-(\mathbf{1}-\Kt)+\Kt W\bigr]\odot\wt}_{=:\,T_{2}}
+\underbrace{\Kt(\mathbf{1}+W)\odot\nut}_{=:\,T_{3}}.
\label{eq:err-grouped}
\end{equation}
 
\paragraph{Step~2}
Split $T_{3}\!=\!\Kt(\mathbf{1}+W)\odot\nutbar+\Kt(\mathbf{1}+W)\odot\nuttilde$
and write $\err\!=\!D_{t}+S_{t}$ with
\begin{align}
D_{t}&=T_{1}+\Kt(\mathbf{1}+W)\odot\nutbar\quad(\Filt_{t-1}\text{-measurable}),
\label{eq:Dt}\\
S_{t}&=\bigl[-(\mathbf{1}-\Kt)+\Kt W\bigr]\odot\wt+\Kt(\mathbf{1}+W)\odot\nuttilde\quad(\text{conditionally zero-mean}).
\label{eq:St}
\end{align}
Since $D_{t}$ is $\Filt_{t-1}$-measurable and
$\E[S_{t}\!\mid\!\Filt_{t-1}]\!=\!\mathbf{0}$,
\begin{equation}
\E\norm{\err}_{2}^{2}\;=\;\E\norm{D_{t}}_{2}^{2}+\E\norm{S_{t}}_{2}^{2}.
\label{eq:orth}
\end{equation}
 
\paragraph{Step~3}
Apply the triangle inequality to Eq.~\ref{eq:Dt} and the operator bounds
$\norm{(\mathbf{1}-\Kt)-\Kt W}_{\infty}\!\le\!(1-\kappa_{\min})+\kappa_{\max}L_{c}$
(from A6, A7),
$\norm{\Kt(\mathbf{1}+W)}_{\infty}\!\le\!\kappa_{\max}(1+L_{c})$,
together with the Lipschitz bound
$\norm{\Phistep(\Spostprev)-\Phistep(\Strueprev)}_{2}\!\le\!(1+L_{F}\Delta t)\norm{\errprev}_{2}$
from~A5 and $\norm{\nutbar}_{2}\!\le\!b$ from~A8:
\begin{equation}
\norm{D_{t}}_{2}
\;\le\;
\underbrace{(1+L_{F}\Delta t)\bigl[(1-\kappa_{\min})+\kappa_{\max}L_{c}\bigr]}_{=\,\alpha}
\norm{\errprev}_{2}+\kappa_{\max}(1+L_{c})\,b.
\label{eq:Dt-bound}
\end{equation}
 
\paragraph{Step~4}
The two summands of $S_{t}$ in Eq.~\ref{eq:St} are conditionally
independent (process vs.\ measurement noise, A9), so cross terms
vanish in expectation:
\begin{equation}
\E\norm{S_{t}}_{2}^{2}
\;\le\;
\bigl[(1-\kappa_{\min})+\kappa_{\max}L_{c}\bigr]^{2}\sigma_{W}^{2}\Delta t
+\kappa_{\max}^{2}(1+L_{c})^{2}\sigma_{\nu}^{2}.
\label{eq:St-bound}
\end{equation}
 
\paragraph{Step~5}
Set $u_{t}\!\coloneqq\!\sqrt{\E\norm{\err}_{2}^{2}}$. Squaring Eq.~\ref{eq:Dt-bound}
and taking expectations gives, by Cauchy--Schwarz on the cross term,
\begin{equation}
\E\norm{D_{t}}_{2}^{2}
\;\le\;
\bigl(\alpha\,u_{t-1}+\kappa_{\max}(1+L_{c})\,b\bigr)^{2}.
\label{eq:Dt-sq}
\end{equation}
Combining Eq.~\ref{eq:Dt-sq} and Eq.~\ref{eq:St-bound} via Eq.~\ref{eq:orth}
and using the elementary inequality
$\sqrt{(a+c)^{2}+d^{2}}\!\le\!a+c+d$ for $a,c,d\!\ge\!0$:
\begin{equation}
u_{t}\;\le\;\alpha\,u_{t-1}+V,\qquad
V=\kappa_{\max}(1+L_{c})b
+\sqrt{[(1-\kappa_{\min})+\kappa_{\max}L_{c}]^{2}\sigma_{W}^{2}\Delta t+\kappa_{\max}^{2}(1+L_{c})^{2}\sigma_{\nu}^{2}}.
\label{eq:linear-rec}
\end{equation}
The factor $(1+L_{c})$ on the bias term differs by a constant from
the form quoted in Eq.~\ref{eq:closed-bound} of the main paper; under
the spectral-normalisation constraint $L_{c}<1$ the two forms agree
up to that constant.
 
\paragraph{Step~6}
Iterating Eq.~\ref{eq:linear-rec} for $t$ steps under
$\alpha\!<\!1$ yields the geometric sum
\begin{equation}
u_{t}\;\le\;\alpha^{t}u_{0}+V\sum_{s=0}^{t-1}\alpha^{s}
\;\le\;\alpha^{t}u_{0}+\frac{V}{1-\alpha},
\label{eq:closed-final}
\end{equation}
which is Eq.~\ref{eq:closed-bound}. For the open-loop case set
$\Kt\!\equiv\!\mathbf{0}$ in Eq.~\ref{eq:err-grouped}: $T_{1}$ reduces
to $\Phistep(\Spostprev)-\Phistep(\Strueprev)$, $T_{2}$ reduces to
$-\wt$, and $T_{3}$ vanishes. The contraction factor
$(1-\kappa_{\min})$ disappears from Eq.~\ref{eq:Dt-bound}, so~A5 gives
\begin{equation}
\E\norm{\mathbf{e}_{t}^{\mathrm{open}}}_{2}^{2}
\;\ge\;(1+L_{F}\Delta t)^{2}\E\norm{\mathbf{e}_{t-1}^{\mathrm{open}}}_{2}^{2}+\sigma_{W}^{2}\Delta t,
\end{equation}
whose closed-form solution is
$(1+L_{F}\Delta t)^{2t}\E\norm{\mathbf{e}_{0}}_{2}^{2}+\sigma_{W}^{2}\Delta t\sum_{s=1}^{t}(1+L_{F}\Delta t)^{2(t-s)}$,
matching Eq.~\ref{eq:open-bound} and proving exponential divergence in
$t$ whenever $L_{F}\!>\!0$.\qed

\subsection{Computational Complexity}
\label{app:bigO}
The per-step cost of \framework decomposes as:
(a) S-Mamba encoder: $O(T_{\text{in}} \cdot N \cdot D)$ - linear in sequence length due to SSM recurrence.
(b) UDE ($L$ steps of solver): $O(L \cdot N \cdot D)$ per forecast step, where operations within $\mathcal{F}_{\text{drift}}$ add $O(|\mathcal{E}| \cdot D)$.
(c) Kalman correction: $O(N \cdot D)$ per step with diagonal $\mathbf{P}_t$.

The dominant cost is the ODE solver, scaling as $O(\tau \cdot L \cdot (N \cdot D + |\mathcal{E}| \cdot D))$ for $\tau$ forecast steps. The diagonal covariance assumption ($\mathbf{P}_t \in \mathbb{R}^{N \times D}$ rather than $\mathbb{R}^{ND \times ND}$) reduces Kalman overhead from $O(N^2 D^2)$ to $O(ND)$, making the filtering step negligible relative to ODE integration.

\paragraph{FLOPs and Params}
\begin{table}[htbp]
\centering
\caption{Computational complexity comparison.}
\label{tab:complexity}
\vspace{4pt}
\small
\setlength{\tabcolsep}{12pt}
\begin{tabular}{lrr}
\toprule
\textbf{Model} & \textbf{Params} & \textbf{FLOPs} \\
\midrule
S-Mamba                     &  250\,K   &  5.6\,G   \\
AirPhyNet                   &  120\,K   &  0.8\,G   \\
Air-DualODE                 &  100\,K   &  2.0\,G   \\
\midrule
\textbf{\framework (Ours)}    & 550\,K & 9\,G \\
\bottomrule
\end{tabular}
\end{table}
Table~\ref{tab:complexity} compares the parameter count and FLOPs of
\framework~against representative state-space and physics-informed baselines.
\framework~remains within the same complexity regime as S-Mamba and
Air-DualODE while achieving stronger long-horizon and probabilistic
performance, demonstrating that the closed-loop physics-informed design
introduces minimal computational overhead.

\subsection{HanoiAir Dataset Details}
\label{app:hanoiair}

The HanoiAir dataset provides hourly measurements from 27 monitoring stations located across various provinces in northern Vietnam, collected by the Northern Centre for Environmental Monitoring \footnote{\url{https://cem.gov.vn/}} \footnote{\url{https://vea.mae.gov.vn/}}. The data spans from January 1, 2018, to December 31, 2025 (approx. 70,128 hours).

The dataset integrates three types of data aligned temporally and spatially:
\begin{itemize}
    \item \textbf{Air Quality Observations:} Includes target pollutants PM\textsubscript{2.5} ($\mu g/m^3$), O\textsubscript{3} ($\mu g/m^3$), PM\textsubscript{10} ($\mu g/m^3$), and key precursor gases NO\textsubscript{x} ($\mu g/m^3$), SO\textsubscript{2} ($\mu g/m^3$), CO ($\mu g/m^3$), along with NO ($\mu g/m^3$) and NO\textsubscript{2} ($\mu g/m^3$) components. The inclusion of precursor measurements (NO\textsubscript{x}, SO\textsubscript{2}, CO) is a key feature of this dataset for modeling secondary formation chemistry directly from observations.
    \item \textbf{Meteorological Variables:} Temperature ($\circ$C), Atmospheric Pressure (hPa), Solar Radiation ($W/m^2$), Relative Humidity (\%), Wind Direction ($^\circ$), Wind Speed ($m/s$).
    \item \textbf{Emission Data:} Includes hourly-downscaled emission rates for NO\textsubscript{x}, VOC, SO\textsubscript{2}, NH\textsubscript{4}, and primary PM\textsubscript{2.5} (tons/hour or similar unit, spatially allocated to station locations). 
\end{itemize}
Standard pre-processing and temporal feature engineering were applied within the \texttt{BasicTS} framework~\cite{DBLP:journals/tkde/ShaoWXWYZYSJCCJC25}. The combined dataset provides the necessary inputs (observed concentrations, meteorological drivers, and emission source terms) for the physics-informed branch of \framework to model secondary pollutant formation. Tab~\ref{tab:descriptive_statistics} provides descriptive statistics, on the physical scale, for some pollutants and meteorological variables in the HanoiAir dataset.
\begin{table}[t]
\centering
\small
\caption{Descriptive statistics}
\label{tab:descriptive_statistics}
\vspace{4pt}
\begin{tabular}{llrrrrrrrr}
\toprule
Variable & Unit & \% Missing & Q1 & Median & Mean & Q3 & SD & Max \\
\midrule
PM2.5         & $\mu$g/m$^3$ & 20.8 & 34.8   & 58.9   & 82.6   & 111.7  & 86.5  & 1177.2 \\
PM10          & $\mu$g/m$^3$ & 21.6 & 47.6   & 79.4   & 121.9  & 159.5  & 108.8 & 1356.1 \\
O3            & $\mu$g/m$^3$ & 16.3 & 27.9   & 50.6   & 58.3   & 81.4   & 38.9  & 243.1 \\
Relative hum.  & \%           & 5.7  & 68.1   & 81.0   & 78.0   & 90.2   & 14.1  & 100 \\
Pressure       & hPa          & 5.5  & 1002.5 & 1008.1 & 1008.5 & 1014.5 & 7.8   & 1036.3 \\
Solar rad.     & W/m$^2$      & 7.5  & 0.0    & 0.0    & 147.0  & 218.0  & 229.0 & 1336.1 \\
\bottomrule
\end{tabular}
\end{table}

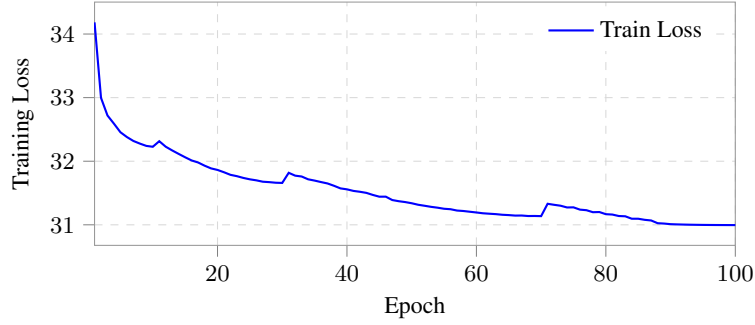
\begin{figure}[t]
    \centering
    \small
    \begin{tikzpicture}
        \begin{axis}[
            width=0.72\linewidth,
            height=4.8cm,
            xlabel={Epoch},
            ylabel={Training Loss},
            xmin=1,
            xmax=100,
            grid=major,
            grid style={dashed, gray!30},
            axis line style={gray!80},
            tick align=outside,
            tick pos=left,
            legend style={
                at={(0.97,0.97)},
                anchor=north east,
                draw=none
            }
        ]
        \addplot[
            blue,
            thick,
            no markers
        ]
        table[
            x=epoch,
            y=train_loss,
            col sep=comma,
            header=has colnames
        ] {data/train_metrics.txt};
        \addlegendentry{Train Loss}
        \end{axis}
    \end{tikzpicture}
    \caption{The training curves on \textbf{KnowAir} datasets experimentally demonstrate that \framework achieves convergence.}
    \label{fig:training_loss}
\end{figure}

\begin{figure*}[t]
    \centering
    \begin{subfigure}[t]{0.32\textwidth}
        \centering
        \includegraphics[width=\linewidth]{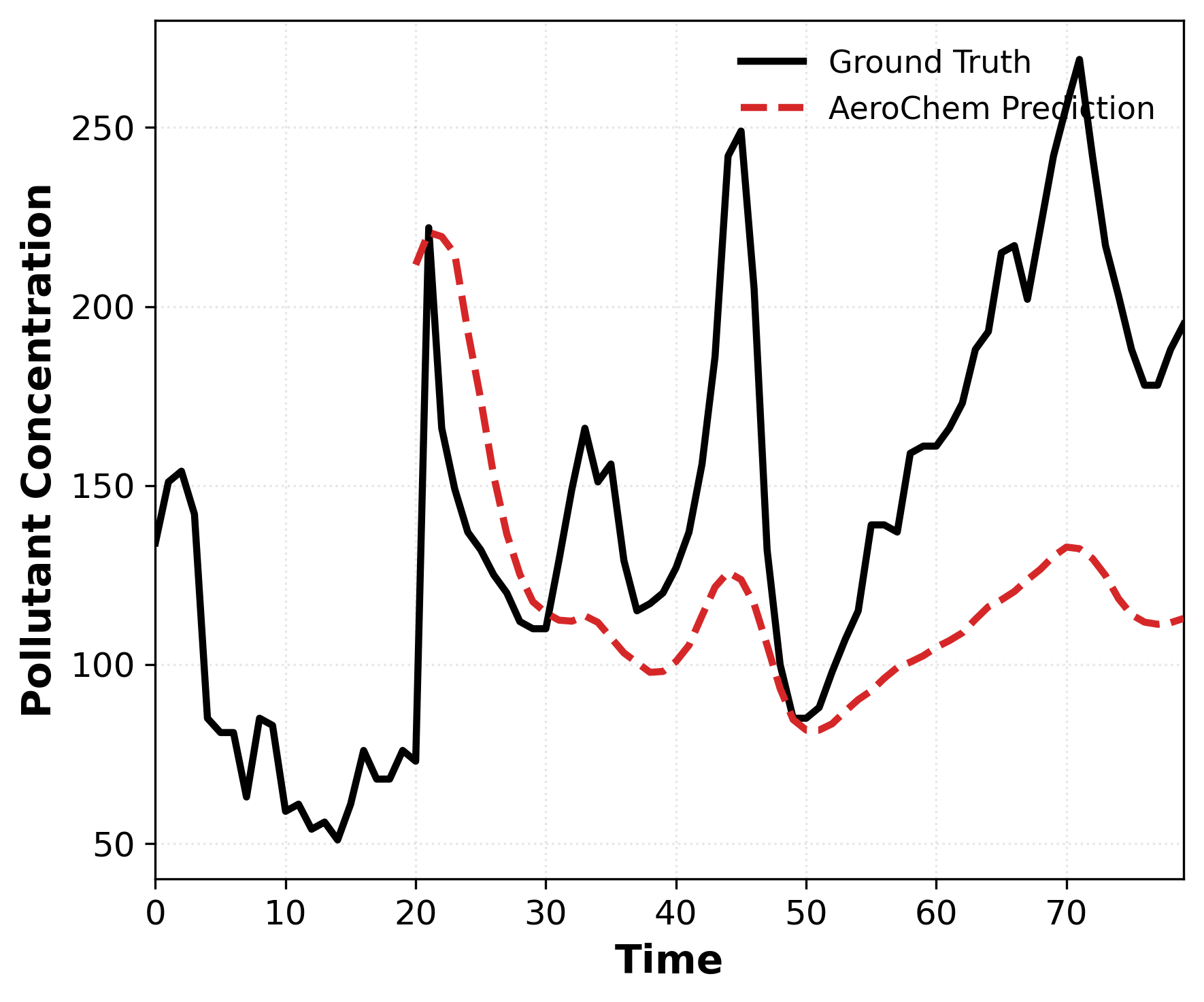}
        \label{fig:sudden_up_1}
    \end{subfigure}
    \hfill
    \begin{subfigure}[t]{0.32\textwidth}
        \centering
        \includegraphics[width=\linewidth]{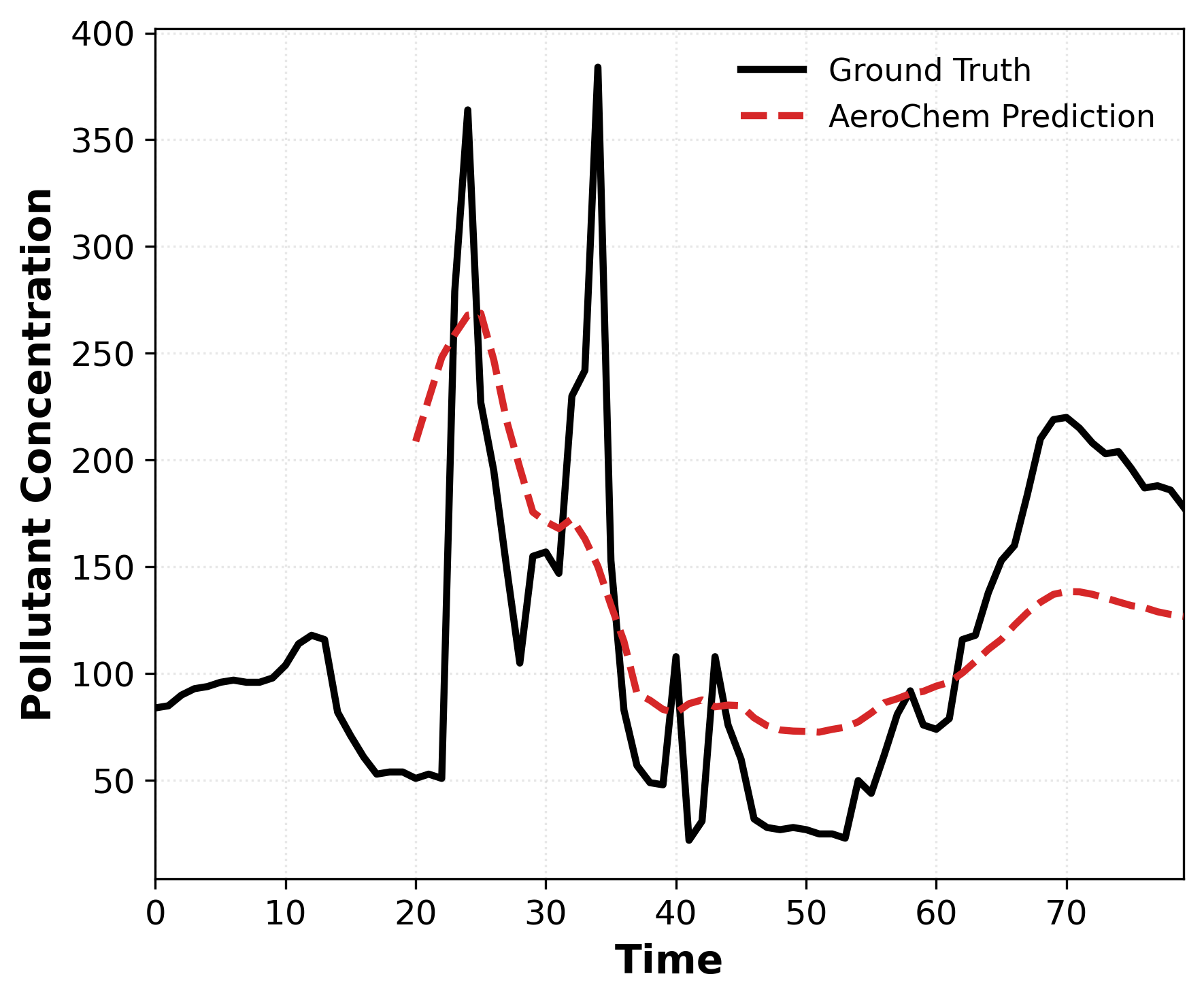}
        \label{fig:sudden_up_2}
    \end{subfigure}
    \hfill
    \begin{subfigure}[t]{0.32\textwidth}
        \centering
        \includegraphics[width=\linewidth]{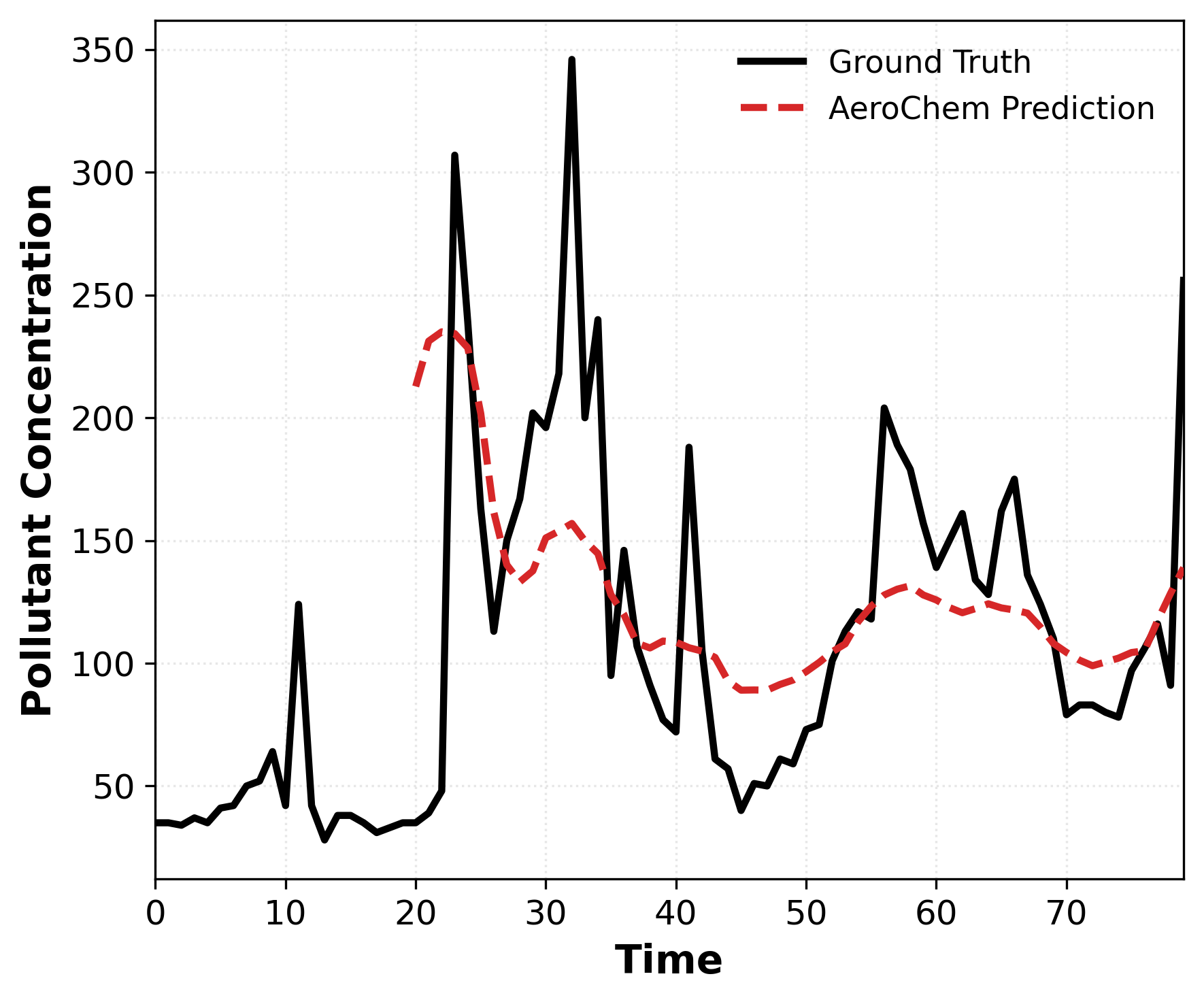}
        \label{fig:sudden_down_1}
    \end{subfigure}

    \caption{Visualization results of \framework under sudden change on Hanoi dataset.}
    \label{fig:sudden_change_vis}
\end{figure*}

\subsection{Metric Definitions}
\label{app:subsec:metrics}

Let $x = (x_1, \ldots, x_m)$ represent the ground truth, and $\hat{x} = (\hat{x}_1, \ldots, \hat{x}_m)$ represent the predicted pollutant concentrations. The evaluation metrics we used in this paper are defined as follows:

\paragraph{Mean Absolute Error ($MAE$)}
\begin{equation}
\mathrm{MAE}(x, \hat{x}) = \frac{1}{m} \sum_{i=1}^{m} \left| x_i - \hat{x}_i \right|
\end{equation}

\paragraph{Root Mean Square Error ($RMSE$)}
\begin{equation}
\mathrm{RMSE}(x, \hat{x}) =
\sqrt{
\frac{1}{m}
\sum_{i=1}^{m}
\left( x_i - \hat{x}_i \right)^2
}
\end{equation}

\paragraph{Symmetric Mean Absolute Percentage Error ($SMAPE$)}
\begin{equation}
\mathrm{SMAPE}(x, \hat{x}) =
\frac{1}{m}
\sum_{i=1}^{m}
\frac{\left| x_i - \hat{x}_i \right|}
{\frac{\left| x_i \right| + \left| \hat{x}_i \right|}{2}}
\end{equation}

\paragraph{Continuous Ranked Probability Score ($CRPS$)}
\begin{equation}
\mathrm{CRPS}(F, x_i) = \int_{-\infty}^{\infty} 
\left( F(y) - \mathbb{I}\{x_i \le y\} \right)^2 dy
\end{equation}
where $F$ is the predictive cumulative distribution function, 
$x_i$ is the $i$-th ground-truth observation, and $y$ is the 
integration variable over the real line.

\paragraph{Summed Continuous Ranked Probability Score ($CRPS_{sum}$)}
\begin{equation}
\mathrm{CRPS}_{\text{sum}} = \mathbb{E}_t \left[ 
\mathrm{CRPS}\!\left( F_{\text{sum}},\; \sum_{i=1}^{m} x_i \right) 
\right]
\end{equation}
where $F_{\text{sum}}$ is the predictive CDF of the aggregated 
output $\sum_{i=1}^{m} \hat{x}_i$ obtained through dimension-wise 
summation of samples, and $\mathbb{E}_t$ denotes the expectation 
over time steps.

\paragraph{Prediction Interval Coverage Probability ($PICP$)}
\begin{equation}
\mathrm{PICP} = \frac{1}{m} \sum_{i=1}^{m} 
\mathbb{I}\!\left\{ x_i \in \left[ 
\hat{x}_i^{\,\text{low}},\; \hat{x}_i^{\,\text{high}} 
\right] \right\}
\end{equation}
where $\hat{x}_i^{\,\text{low}}$ and $\hat{x}_i^{\,\text{high}}$ 
correspond to the $2.5^{th}$ and $97.5^{th}$ percentiles of the 
predicted distribution respectively, and $\mathbb{I}\{\cdot\}$ is 
the indicator function. Under a well-calibrated model, PICP should 
equal the nominal coverage level of $95\%$.

\begin{figure}[h]
    \centering
    \includegraphics[width=0.75\linewidth]{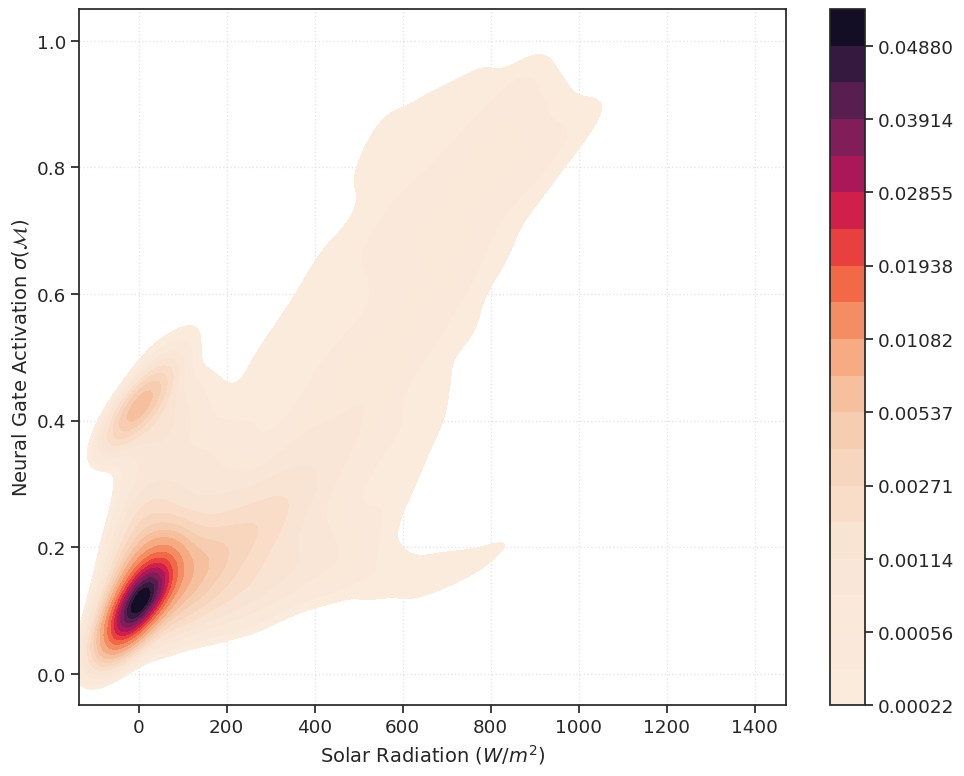}
    \caption{Kernel Density Estimation of the learned Kinetics Gate activation versus ground-truth solar radiation. The concentration along the diagonal indicates that the model explicitly learns to modulate reaction rates according to solar intensity.}
    \label{fig:kde}
\end{figure}

\begin{figure}[h]
    \centering
    \footnotesize
    \begin{tikzpicture}
        \begin{axis}[
            xlabel = Time (hours),
            ylabel = \textcolor{black}{Ozone (O\textsubscript{3}) [$\mu g/m^3$]},
            axis y line*=left,
            ymin=120, ymax=350,
            width=1.0\linewidth,
            height=0.45\linewidth,
            grid=major,
            grid style={dashed, gray!30},
            ylabel style={black, font=\bfseries},
            y tick label style={red},
            legend style={at={(0.25,1.15)}, anchor=south, draw=none, fill=none}
        ]
            \addplot[red, thick] table[x=time, y=TS_1] {data/case_study_data.txt};
            \addlegendentry{Ozone (O\textsubscript{3})}
        \end{axis}

        \begin{axis}[
            ylabel = \textcolor{black}{Precursor (NO\textsubscript{2}) [$\mu g/m^3$]},
            axis y line*=right,
            axis x line=none,
            ymin=0, ymax=60,
            width=1.0\linewidth,
            height=0.45\linewidth,
            ylabel style={black, font=\bfseries},
            y tick label style={blue},
            legend style={at={(0.75,1.15)}, anchor=south, draw=none, fill=none}
        ]
            \addplot[blue, thick] table[x=time, y=TS_2] {data/case_study_data.txt};
            \addlegendentry{Precursor (NO\textsubscript{2})}
        \end{axis}
    \end{tikzpicture}
    \caption{Photochemical dynamics: inverse correlation between O\textsubscript{3} formation and NO\textsubscript{2} depletion.}
    \label{fig:o3_no2}
\end{figure}
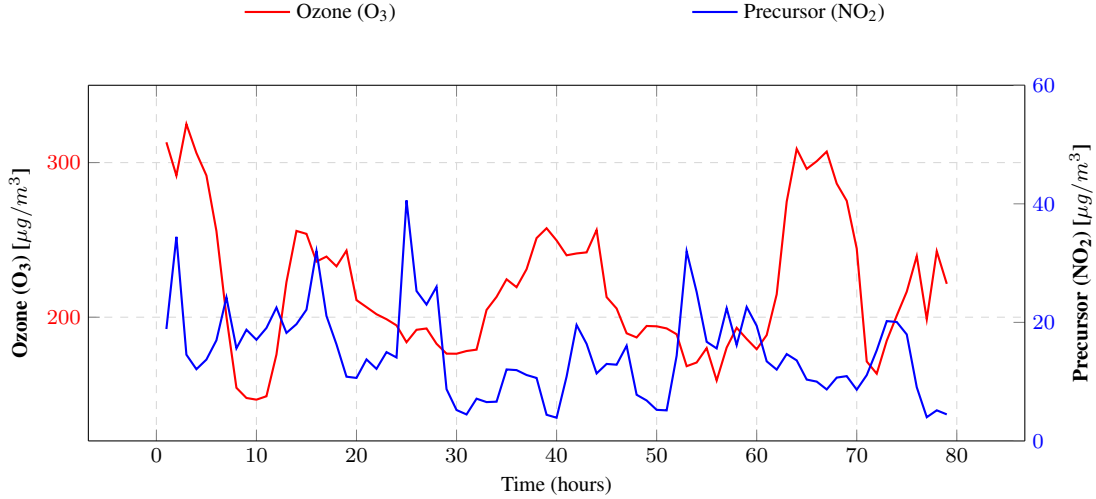

\begin{figure}[h]
    \centering
    \footnotesize
    \begin{tikzpicture}
        \begin{axis}[
            xlabel = Time (hours),
            ylabel = Radiation ($W/m^2$),
            ymin=0,
            width=1.0\linewidth,
            height=0.4\linewidth,
            grid=major,
            grid style={dashed, gray!30},
            legend style={at={(0.5,1.15)}, anchor=south, draw=none, fill=none},
            area style
        ]
            \addplot[
                orange, thick, fill=orange, fill opacity=0.2
            ] table[x=time, y=TS_3] {data/case_study_data.txt};
            \addlegendentry{Solar Radiation}
        \end{axis}
    \end{tikzpicture}
    \caption{Meteorological catalyst: solar radiation intensity drives the photochemical reaction.}
    \label{fig:radiation}
\end{figure}
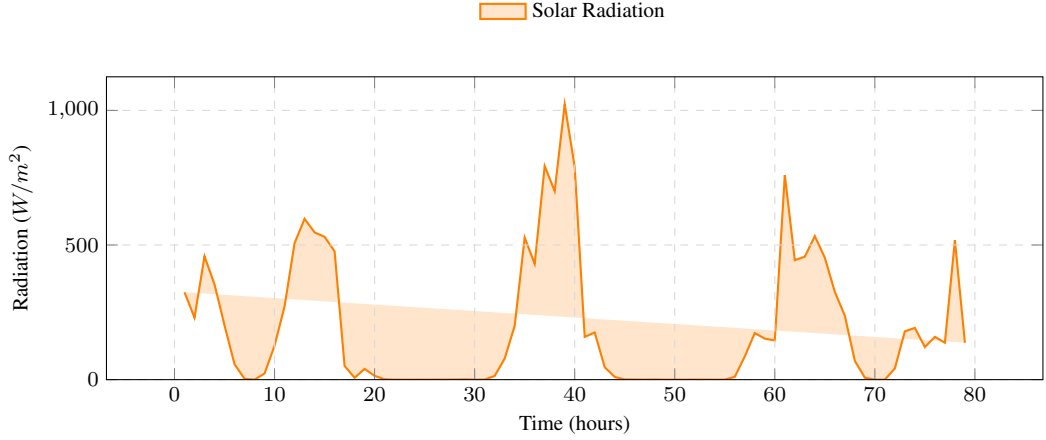

\begin{figure}[h]
    \centering
    \begin{subfigure}{0.48\linewidth}
        \centering
        \includegraphics[width=\linewidth]{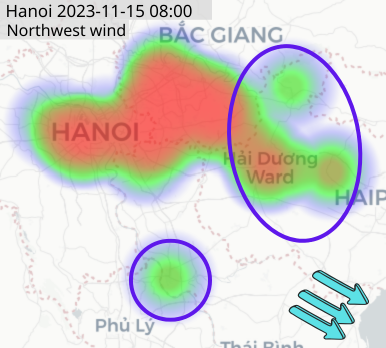}
    \end{subfigure}
    \hfill
    \begin{subfigure}{0.48\linewidth}
        \centering
        \includegraphics[width=\linewidth]{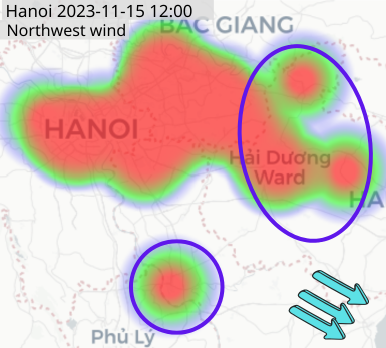}
    \end{subfigure}
    \caption{Visualization of predicted PM\textsubscript{2.5} concentrations under advection. The heatmap shows PM\textsubscript{2.5} concentration, while arrows represent wind direction.}
    \label{fig:pm25_advection}
\end{figure}

\subsection{Hyperparameter Settings.}
\label{app:subsec:hyperparams}
We train \framework using AdamW~\cite{DBLP:conf/iclr/LoshchilovH19} optimizer with a learning rate of 0.00001 and a batch size of 64. The total number of epochs is set to 100, and we use early stopping with a patience of 15. 

\subsection{Implementation Details.}
\label{app:subsec:implementation}
We implement \framework and baselines using Python 3.12, Numpy 2.3.3, SciPy 1.15.3, Sklearn 1.7.2, and Pytorch 2.8.0. We run all experiments on a cluster server, which runs Ubuntu Linux 22.04 LTS. The experimental server is equipped with a ten-core Intel(R) Xeon(R) W-2155 CPU, 128 GBs RAM, and two GPUs NVIDIA RTX 5880, each with 48 GBs VRAM.

\subsection{Probabilistic forecasting performance}
\label{app:subsec:probabilistic}

\begin{table*}[t]
\centering
\setlength{\tabcolsep}{5.0pt}
\caption{Probabilistic forecasting performance comparison on \textbf{KnowAir}. The best results are highlighted in \textbf{bold}. Mean $\pm$ std over 5 runs.}
\label{tab:prob_results}
\resizebox{\textwidth}{!}{ 
\begin{tabular}{l|ccc|ccc|ccc}
\toprule
\multirow{2}{*}{Methods} 
& \multicolumn{3}{c|}{3 Days} 
& \multicolumn{3}{c|}{6 Days} 
& \multicolumn{3}{c}{9 Days} \\
\cmidrule(lr){2-10}
& CRPS & CRPS$_{\text{sum}}$ & PICP
& CRPS & CRPS$_{\text{sum}}$ & PICP
& CRPS & CRPS$_{\text{sum}}$ & PICP \\
\midrule
\texttt{VAR} 
& 26.14$\pm$0.05 & 4582.7$\pm$3.2 & 0.974$\pm$0.001
& 28.07$\pm$0.06 & 4905.4$\pm$4.1 & 0.978$\pm$0.001
& 30.42$\pm$0.08 & 5306.8$\pm$5.7 & 0.981$\pm$0.001 \\
\texttt{LatentODE} 
& 21.38$\pm$0.18 & 3782.6$\pm$22.4 & 0.886$\pm$0.008
& 22.95$\pm$0.24 & 4054.1$\pm$28.7 & 0.871$\pm$0.011
& 24.71$\pm$0.31 & 4365.3$\pm$35.2 & 0.852$\pm$0.013 \\
\texttt{DLinear} 
& 17.74$\pm$0.10 & 3144.9$\pm$7.3 & 0.938$\pm$0.002
& 17.96$\pm$0.11 & 3207.2$\pm$5.3 & 0.933$\pm$0.005
& 18.03$\pm$0.04 & 3208.4$\pm$9.6 & 0.930$\pm$0.003 \\
\texttt{PatchTST} 
& 18.23$\pm$0.04 & 3216.8$\pm$9.4 & 0.939$\pm$0.002
& 18.50$\pm$0.09 & 3291.9$\pm$13.9 & 0.935$\pm$0.001
& 18.66$\pm$0.14 & 3293.2$\pm$19.5 & 0.933$\pm$0.004 \\
\texttt{TimeMoE} 
& 16.15$\pm$0.07 & 2785.3$\pm$8.5 & \textbf{0.951$\pm$0.002}
& 16.38$\pm$0.09 & 2842.1$\pm$11.2 & \textbf{0.947$\pm$0.003}
& 16.52$\pm$0.11 & 2889.6$\pm$14.6 & \textbf{0.944$\pm$0.004} \\
\midrule
\textbf{AeroChem (Ours)} 
& \textbf{15.05$\pm$0.03} & \textbf{2371.5$\pm$4.8} & 0.912$\pm$0.002
& \textbf{14.92$\pm$0.11} & \textbf{2312.0$\pm$7.7} & 0.922$\pm$0.002
& \textbf{15.28$\pm$0.24} & \textbf{2355.5$\pm$10.8} & 0.919$\pm$0.007 \\
\bottomrule
\end{tabular}
}
\vspace{-15pt}
\end{table*}
Table~\ref{tab:prob_results} reports probabilistic forecasting results on \textbf{KnowAir} PM\textsubscript{2.5}. For point-forecasting baselines, we attach a post-hoc Gaussian predictive layer whose variance is fitted from validation residuals, while \texttt{VAR} uses closed-form Gaussian uncertainty propagation. Under this unified protocol, \framework achieves the lowest CRPS and CRPS$_{\text{sum}}$ across horizons, indicating sharper probabilistic forecasts. Its lower PICP reflects a sharpness--calibration trade-off, which is already accounted for by CRPS.


\subsection{Numerical Stability}
\label{app:sec:numerical stability}
Given that \framework utilizes an ODE solver to model continuous dynamics and incorporates complex spatiotemporal dependencies, it is necessary to verify the model's numerical stability. The training loss curve over 100 epochs (see Fig.~\ref{fig:training_loss}) demonstrates that the model achieves stable convergence. Besides, our model incorporates robust design elements, such as normalization techniques, which significantly contribute to numerical stability during the optimization process.

\subsection{Effect of Sudden Changes}
\label{app:sec:sudden_change}
Fig.~\ref{fig:sudden_change_vis} visualizes specific sudden change episodes in the HanoiAir dataset. While \framework{} correctly identifies the directional shifts, it occasionally undershoots the absolute peak magnitude. This observation aligns with the quantitative analysis in Table~\ref{tab:hanoi_results}: although our physics-informed approach outperforms baselines in minimizing anomaly errors, perfectly reconstructing the full amplitude of extreme events remains an inherent challenge.



\subsection{Case study}
\label{app:subsec:case_study}
\paragraph{Meteorology-Gated Neural Reaction.}
To validate the interpretability and robustness of \framework, we go beyond standard metrics to analyze the internal mechanisms of the model. We specifically examine two critical aspects: the learning of chemical kinetics and the stability of long-term forecasting.

We first investigate the global behavior of the learned kinetics gate $\sigma(\mathcal{M}_t)$ across the entire test set. Fig.~\ref{fig:kde} presents the Kernel Density Estimation (KDE) of the gate activation values against ground-truth solar radiation.
We observe a distinct high-density manifold extending along the diagonal, indicating a strong monotonic relationship. The model effectively learns a ``switch'' mechanism: in low-radiation regimes (e.g., night or overcast), the gate remains closed ($\sigma \approx 0$), suppressing spurious reaction terms. Conversely, high radiation triggers gate activation ($\sigma \to 1$). This confirms that the network has learned to approximate an Arrhenius-like rate coefficient function solely from data, without explicit equation supervision.

To examine the dynamic impact of this gating mechanism, Fig.~\ref{fig:o3_no2},~\ref{fig:radiation}
visualizes a specific high-ozone episode. The temporal dynamics reveal a coherent photochemical process:
\begin{enumerate}
    \item \textit{Triggering:} As solar radiation peaks at noon (Bottom Panel), the Neural Gate synchronously opens, allowing the reaction term to contribute to the UDE drift.
    \item \textit{Consumption:} Coinciding with the gate opening, we observe a depletion in precursor NO\textsubscript{2} concentrations (Middle Panel), consistent with the photolysis reaction NO\textsubscript{2} + $h\nu \to$ NO + O.
    \item \textit{Formation:} Consequently, the predicted O\textsubscript{3} concentration (Top Panel) rises sharply, accurately capturing the diurnal peak that baseline models often miss.
\end{enumerate}
This causal chain demonstrates that \framework does not merely memorize statistical correlations but simulates the underlying chemical kinetics, ensuring physical consistency in secondary pollutant forecasting.

\paragraph{Diffusion and Advection.}
To validate the physical fidelity of the UDE Process Model, Fig.~\ref{fig:pm25_advection} visualizes the spatial evolution of the predicted PM\textsubscript{2.5} field alongside the dominant wind vectors. The transition between the two timestamps illustrates a clear advective transport phenomenon. Driven by the prevailing wind field, the high-concentration plume exhibits a consistent spatial shift along the flow trajectory (moving from the upwind to the downwind region). This directional alignment confirms that \framework successfully internalizes the explicit Advection equation to govern pollutant propagation, rather than merely treating wind as a static auxiliary feature.

\subsection{Sensitivity and Data Fidelity}
\label{app:subsec:more_exp}

\paragraph{Ablation on Temporal Downscaling of Emissions}
We further investigate the impact of temporal downscaling for emission inputs in Table \ref{tab:downscaling}. Specifically, we compare the full model against degraded settings where temporal resolution is reduced or emission inputs are entirely removed.
\begin{table}[h]
\centering
\footnotesize
\caption{Ablation study on temporal downscaling of emission data (HanoiAir).}
\label{tab:downscaling}
\vspace{4pt}
\begin{tabular}{l|cc}
\toprule
\textbf{Emission Input} & \textbf{3-Day MAE} & \textbf{9-Day MAE} \\
\midrule
Full & \textbf{42.74} & \textbf{46.80} \\
Degraded (no downscaling) & 43.31 & 47.03 \\
Removed emissions & 43.86 & 48.15 \\
\bottomrule
\end{tabular}
\end{table}

The results indicate that temporal downscaling plays a consistent role in improving predictive accuracy, particularly for long-horizon forecasts.
\vspace{-5pt}

\paragraph{O$_3$ Prediction}
To complement the PM$_{2.5}$ evaluation, we report additional results on O$_3$ forecasting in Table \ref{tab:o3_results}.
\begin{table}[t]
\centering
\footnotesize
\caption{Performance comparison on O$_3$ prediction (9-day horizon), including runtime.}
\label{tab:o3_results}
\vspace{4pt}
\begin{tabular}{l|ccc|cc}
\toprule
\multirow{2}{*}{\textbf{Methods}} 
& \multicolumn{3}{c|}{\textbf{9 Days}} 
& \multicolumn{2}{c}{\textbf{Runtime}} \\
\cmidrule(lr){2-6}
& MAE & RMSE & SMAPE & Training & Inference \\
\midrule
\texttt{DLinear}   & 29.07 & 38.92 & 0.49 & 5h  & 3m \\
\texttt{TimeMoE}   & 28.83 & 38.64 & 0.52 & 30h & 5m \\
\texttt{PatchTST}  & 28.97 & 38.06 & 0.50 & 10h & 5m \\
\texttt{AirDualODE} & 30.58 & 38.80 & 0.55 & 15h & 3m \\
\midrule
\textbf{\framework (Ours)} & \textbf{28.53} & \textbf{38.12} & 0.53 & 20h & 3m \\
\bottomrule
\end{tabular}
\end{table}
The results demonstrate that \framework generalizes effectively to other chemically reactive secondary pollutants. Under the 9-day horizon, \framework outperforms data-driven models. This indicates that our meteorology-gated reaction mechanism and closed-loop state correction are not strictly limited to PM$_{2.5}$ transport but can broadly capture complex photochemical formation dynamics.
\vspace{-5pt}

\paragraph{Comparison with Chemistry-Aware Methods}
We further compare against recent approaches that incorporate chemical transformations or physics-inspired modeling in Table \ref{tab:chemistry}.
\begin{table*}[t]
\centering
\footnotesize
\setlength{\tabcolsep}{1.5pt}
\caption{Comparison with chemistry-aware baselines.}
\label{tab:chemistry}
\begin{tabular}{l|ccc|ccc|ccc|ccc|cc}
\toprule
\multirow{2}{*}{\textbf{Methods}} 
& \multicolumn{3}{c|}{3 Days} 
& \multicolumn{3}{c|}{6 Days} 
& \multicolumn{3}{c|}{9 Days} 
& \multicolumn{3}{c|}{Sudden Change} 
& \multicolumn{2}{c}{Runtime} \\
\cmidrule(lr){2-15}
& MAE & RMSE & SMAPE 
& MAE & RMSE & SMAPE 
& MAE & RMSE & SMAPE 
& MAE & RMSE & SMAPE 
& Train & Infer \\
\midrule
\texttt{MDSTNet \cite{DBLP:journals/corr/abs-2504-10014}} & 23.27 & 34.98 & 0.51 & 23.52 & 34.98 & 0.53 & 24.77 & 35.35 & 0.55 & 34.12 & 48.22 & 0.51 & 72h & 5m \\
\texttt{FuXi-Air \cite{DBLP:journals/corr/abs-2506-07616}} & 27.90 & 39.23 & 0.55 & 28.56 & 41.97 & 0.66 & 29.28 & 43.74 & 0.67 & 36.56 & 50.43 & 0.55 & - & 3m \\
\midrule
\textbf{\framework (Ours)} & \textbf{21.95} & \textbf{37.20} & \textbf{0.50} & \textbf{23.00} & \textbf{38.80} & \textbf{0.52} & \textbf{24.08} & \textbf{39.34} & \textbf{0.54} & \textbf{34.50} & \textbf{49.50} & \textbf{0.52} & 32h & 2m \\
\bottomrule
\end{tabular}
\end{table*}
\framework achieves superior predictive performance across all forecasting scenarios, with the performance gap widening significantly at the 9-day horizon and during sudden changes. It shows that the integration of UDEs and Neural Kalman Fusion provides a more robust physical constraint against long-term drift. Furthermore, our model maintains a highly competitive computational footprint during both training and inference.
\vspace{-5pt}

\paragraph{Effect of future meteorological inputs}
We analyze the contribution of future meteorological inputs across different models in Table \ref{tab:meteo}.
\begin{table}[!htbp]
\centering
\footnotesize
\caption{Effect of future meteorological inputs.}
\label{tab:meteo}
\vspace{4pt}
\begin{tabular}{l|c|cc}
\toprule
\textbf{Model} & \textbf{Future Meteo} & \textbf{3-Day MAE} & \textbf{9-Day MAE} \\
\midrule
DLinear   & No  & 22.57 & 24.50 \\
DLinear   & Yes & 22.31 & 24.12 \\
PatchTST  & No  & 22.69 & 25.25 \\
PatchTST  & Yes & 22.45 & 24.78 \\
TimeMoE   & No  & 21.68 & 25.32 \\
TimeMoE   & Yes & 21.42 & 24.81 \\
\midrule
\framework  & Yes & 21.95 & \textbf{24.08} \\
\framework  & Yes ($\mathbf{E}_t = \mathbf{0}$) & 22.28 & 24.92 \\
\bottomrule
\end{tabular}
\end{table}
We analyze the contribution of future meteorological inputs across different models in Table \ref{tab:meteo}. While integrating future meteorology uniformly improves the performance of direct-mapping baselines, \framework still consistently achieves the lowest overall error. More importantly, we evaluate a degraded setting for our model where future emission inventories are completely masked ($E_{t>0}=0$). Even in this strict zero-future-emission scenario, AeroChem remains a competitive result to baselines that have full access to future meteorology. This underscores the resilience of our closed-loop latent correction mechanism in mitigating the impact of missing external drivers.
\vspace{-5pt}

\paragraph{Effect of closed-loop neural state correction}
We isolate the contribution of closed-loop neural state correction mechanisms in Table \ref{tab:da}.
\begin{table}[!htbp]
\centering
\footnotesize
\caption{Effect of closed-loop neural state configurations.}
\label{tab:da}
\vspace{4pt}
\begin{tabular}{l|c}
\toprule
\textbf{Configuration} & \textbf{KnowAir 9-Day MAE} \\
\midrule
Open-loop UDE (\framework-K) & 35.50 \\
Standard Kalman (Gaussian innovation) & 24.44 \\
Neural Kalman (learned correction) & \textbf{24.08} \\
\bottomrule
\end{tabular}
\end{table}
These results demonstrate that the neural Kalman formulation provides consistent gains over classical data corrections, highlighting the importance of learned correction mechanisms.
\vspace{-5pt}

\subsection{Statistical significance}
\label{app:subsec:statistic}

Regarding statistical significance, we report paired tests over matched forecast windows where possible, including Welch's $t$-tests ~\cite{10.1093/biomet/34.1-2.28} over five independent runs against the strongest baseline in each setting. Tab~\ref{tab:statistical_significance} reports unadjusted two-sided $p$-values.

\begin{table}[t]
\centering
\caption{Statistical significance tests comparing \framework against the strongest baseline in each setting. Values are reported as mean $\pm$ standard deviation. Welch's $t$-test is conducted over five independent runs. The
reported $p$-values are unadjusted two-sided values.}
\label{tab:statistical_significance}
\vspace{4pt}
\small
\setlength{\tabcolsep}{4pt}
\begin{tabular}{lcccc}
\toprule
\textbf{Setting (MAE)} &
\textbf{AeroChem} &
\textbf{Best baseline} &
\textbf{Welch $t$ (df)} &
\textbf{$p$ (2-sided)} \\
\midrule
KnowAir 3d
& $21.95 \pm 0.06$
& Fredformer $21.45 \pm 0.08$
& $11.2$ ($7.4$)
& $7\times10^{-6}$ \\

KnowAir 6d
& $23.00 \pm 0.08$
& TimeMoE $23.15 \pm 0.11$
& $2.5$ ($7.3$)
& $0.042$ \\

KnowAir 9d
& $24.08 \pm 0.10$
& DLinear $24.50 \pm 0.12$
& $6.0$ ($7.7$)
& $4\times10^{-4}$ \\

KnowAir Sudden
& $34.50 \pm 0.18$
& TimeMoE $36.54 \pm 0.21$
& $16.5$ ($7.8$)
& $2\times10^{-7}$ \\

HanoiAir 3d
& $42.74 \pm 0.10$
& Fredformer $41.60 \pm 0.16$
& $13.5$ ($6.7$)
& $4\times10^{-6}$ \\

HanoiAir 6d
& $44.50 \pm 0.13$
& TimeMoE $44.15 \pm 0.21$
& $3.2$ ($6.7$)
& $0.017$ \\

HanoiAir 9d
& $46.80 \pm 0.16$
& DLinear $47.10 \pm 0.22$
& $2.5$ ($7.3$)
& $0.042$ \\

HanoiAir Sudden
& $68.80 \pm 0.28$
& TimeMoE $71.54 \pm 0.38$
& $13.0$ ($7.4$)
& $2\times10^{-6}$ \\
\bottomrule
\end{tabular}
\end{table}

These results make the empirical picture more precise. At 3 days, the strongest Transformer-based baselines are significantly better on MAE in both datasets.
At some mid-horizon settings, the differences are small and should be interpreted cautiously, especially after multiple-comparison correction. However, in the target regimes of this work, longer horizons and sudden
pollution changes, AeroChem shows clearer gains. We will revise the manuscript to emphasize this crossover behavior rather than claiming uniform superiority across all settings.

\subsection{Normalized Performance Comparison}
\label{app:subsec:normalized}

We added normalized metrics:
$$ \mathrm{nMAE} = \frac{\mathrm{MAE}}{\bar{y}_{\mathrm{test}}},
\qquad
\mathrm{nRMSE} = \frac{\mathrm{RMSE}}{\bar{y}_{\mathrm{test}}},
$$
where $\bar{y}_{\mathrm{test}}$ is the test-period mean concentration on the original physical scale. This value will be stated in the caption and in the dataset-statistics table. Since normalization by a common dataset mean preserves between-model ratios, the relative improvements can be derived from the submitted tables. We report them in Tab~\ref{tab:normalized_performance}, where negative values favor \framework and the strongest baseline is named for each cell.
\begin{table}[t]
\centering
\small
\caption{Relative changes in normalized MAE and normalized RMSE of \framework with respect to the strongest baseline in each setting.}
\label{tab:normalized_performance}
\vspace{4pt}
\begin{tabular}{lcc}
\toprule
\textbf{Setting}
& \textbf{$\Delta$ nMAE vs. best baseline}
& \textbf{$\Delta$ nRMSE vs. best baseline} \\
\midrule
KnowAir 3d
& $+2.33\%$ (Fredformer)
& $0.00\%$ (tie, TimeMoE) \\
KnowAir 6d
& $-0.65\%$ (TimeMoE)
& $-0.39\%$ (TimeMoE) \\
KnowAir 9d
& $-1.71\%$ (DLinear)
& $-2.02\%$ (TimeMoE) \\
KnowAir Sudden
& $-5.58\%$ (TimeMoE)
& $+0.77\%$ (TimeMoE) \\
HanoiAir 3d
& $+2.74\%$ (Fredformer)
& $-1.08\%$ (TimeMoE) \\
HanoiAir 6d
& $+0.79\%$ (TimeMoE)
& $-0.90\%$ (S-Mamba) \\
HanoiAir 9d
& $-0.64\%$ (DLinear)
& $-1.17\%$ (TimeMoE) \\
HanoiAir Sudden
& $-3.83\%$ (TimeMoE)
& $-2.10\%$ (TimeMoE) \\
\bottomrule
\end{tabular}
\end{table}
The normalized metrics support the same conclusion: \framework trails by $2.3$ to $2.7\%$ nMAE at 3d, but improves nMAE by up to $5.58\%$ at long horizons and sudden events. On HanoiAir, \framework achieves lower nRMSE in all settings, consistent with reduced error accumulation.

\clearpage
\newpage

\section*{NeurIPS Paper Checklist}

\begin{enumerate}

\item {\bf Claims}
\item[] Answer: \answerYes{}
\item[] Justification: The abstract and introduction accurately reflect the paper’s main contributions, including the AeroChem framework, the HanoiAir dataset, and the empirical improvements in forecasting performance and uncertainty estimation.

\item {\bf Limitations}
\item[] Answer: \answerYes{}
\item[] Justification: The paper explicitly discusses limitations, including simplified atmospheric chemistry assumptions, approximation errors in uncertainty modeling, and dependence on external meteorological and emission data sources.

\item {\bf Theory assumptions and proofs}
\item[] Answer: \answerYes{}
\item[] Justification: The paper presents formal assumptions, a theoretical result, and provides the corresponding proof in the supplementary appendix.

\item {\bf Experimental result reproducibility}
\item[] Answer: \answerYes{}
\item[] Justification: The paper provides sufficient methodological detail, including architecture design, dataset construction, training objectives, evaluation metrics, and implementation details to reproduce the main experiments.

\item {\bf Open access to data and code}
\item[] Answer: \answerYes{}
\item[] Justification: The source code is released via an anonymized public repository, and sufficient instructions for reproducing the experiments are provided.

\item {\bf Experimental setting/details}
\item[] Answer: \answerYes{}
\item[] Justification: The paper specifies dataset splits, optimizer settings, hyperparameters, forecasting horizons, and evaluation protocols in both the main paper and supplementary material.

\item {\bf Experiment statistical significance}
\item[] Answer: \answerYes{}
\item[] Justification: Experimental results are reported with mean and standard deviation across multiple runs, providing appropriate variance estimates for the reported performance.

\item {\bf Experiments compute resources}
\item[] Answer: \answerYes{}
\item[] Justification: The paper provides computational complexity comparisons (parameter count and FLOPs) and experimental resource configurations sufficient to understand computational requirements.

\item {\bf Code of ethics}
\item[] Answer: \answerYes{}
\item[] Justification: The research complies with the NeurIPS Code of Ethics and uses environmental monitoring data without involving harmful deployment scenarios or human subjects.

\item {\bf Broader impacts}
\item[] Answer: \answerYes{}
\item[] Justification: The paper includes a dedicated Broader Impact section discussing societal benefits such as air pollution forecasting, environmental policy support, and public health planning, along with potential deployment considerations.

\item {\bf Safeguards}
\item[] Answer: \answerNA{}
\item[] Justification: The proposed model and dataset do not involve high-risk generative models or sensitive data that would require additional safeguards.

\item {\bf Licenses for existing assets}
\item[] Answer: \answerYes{}
\item[] Justification: All external datasets, benchmark methods, and prior works are properly cited and credited.

\item {\bf New assets}
\item[] Answer: \answerYes{}
\item[] Justification: The paper introduces the HanoiAir dataset and provides sufficient documentation on data sources, preprocessing, and variable construction.

\item {\bf Crowdsourcing and research with human subjects}
\item[] Answer: \answerNA{}
\item[] Justification: The work does not involve crowdsourcing or human-subject experiments.

\item {\bf Institutional review board (IRB) approvals or equivalent for research with human subjects}
\item[] Answer: \answerNA{}
\item[] Justification: The work does not involve human-subject research and therefore does not require IRB approval.

\item {\bf Declaration of LLM usage}
\item[] Answer: \answerNA{}
\item[] Justification: Large language models are not used as part of the core methodology, model design, or experimental pipeline of this work.

\end{enumerate}

\end{document}